%% file: iclr2024_conference.tex
\documentclass{article} 
\usepackage{iclr2024_conference,times}

\input{math_commands.tex}

\input{macros}

\newcommand{\revised}[1]{#1}
\newcommand{\addition}[1]{#1}
\newcommand{\highlight}[1]{#1}
\newcommand{\cameraready}[1]{\textcolor{black}{#1}}

\usepackage{booktabs}
\usepackage{marginnote}

\usepackage{hyperref}

\usepackage{url}

\input{packages}
\title{Neural structure learning with stochastic differential equations}

\author{Benjie Wang\thanks{Equal contribution.}  \ \thanks{Correspondence to \texttt{benjiewang@ucla.edu, wenbogong@microsoft.com}} \ \thanks{Work done during an internship at Microsoft Research Cambridge.}\\
University of California, Los Angeles 
\And
Joel Jennings \\
DeepMind 
\And
Wenbo Gong\footnotemark[1] \ \footnotemark[2]\\
Microsoft Research Cambridge
}

\iclrfinalcopy 
\begin{document}

\maketitle

\begin{abstract}
Discovering the underlying relationships among variables from temporal observations has been a longstanding challenge in numerous scientific disciplines, including biology, finance, and climate science. The dynamics of such systems are often best described using continuous-time stochastic processes. Unfortunately, most existing structure learning approaches assume that the underlying process evolves in discrete-time and/or observations occur at regular time intervals.
These mismatched assumptions can often lead to incorrect learned structures and models. 
In this work, we introduce a novel structure learning method, \ModelName{}, which combines neural stochastic differential equations (SDE) with variational inference to infer a posterior distribution over possible structures. 
This continuous-time approach can naturally handle both learning from and predicting observations at arbitrary time points.
Theoretically, we establish sufficient conditions for an SDE and \ModelName{} to be structurally identifiable, and prove its consistency under infinite data limits.
Empirically, we demonstrate that our approach leads to improved structure learning performance on both synthetic and real-world datasets compared to relevant baselines under regular and irregular sampling intervals. 

\end{abstract}

\input{Introduction}
\input{Preliminary}

\input{Methodology}

\input{Theory_main}

\input{Related_work}
\input{Experiments}
\input{Conclusion}

\clearpage
\input{Acknowledgements}

\bibliography{iclr2024_conference}
\bibliographystyle{iclr2024_conference}

\clearpage
\appendix
\input{Theory}
\input{Appendix_model_arch}
\input{Baselines}
\input{Appendix_experiments}
\input{Appendix_interventions}
\end{document}

%% file: math_commands.tex
\usepackage{amsmath,amsfonts,bm, amsthm}
\usepackage{tikz}
\usepackage{algorithm}
\usepackage{algorithmic}
\usepackage{thmtools}
\usepackage{thm-restate}

\def\eqref#1{equation~\ref{#1}}

\def\1{\bm{1}}

\def\vmu{{\bm{\mu}}}

\def\va{{\bm{a}}}

\def\vc{{\bm{c}}}

\def\ve{{\bm{e}}}
\def\vf{{\bm{f}}}
\def\vg{{\bm{g}}}
\def\vh{{\bm{h}}}

\def\vr{{\bm{r}}}

\def\vu{{\bm{u}}}

\def\vx{{\bm{x}}}
\def\vy{{\bm{y}}}

\def\mA{{\bm{A}}}

\def\mG{{\bm{G}}}
\def\mH{{\bm{H}}}
\def\mI{{\bm{I}}}

\def\mK{{\bm{K}}}

\def\mM{{\bm{M}}}

\def\mW{{\bm{W}}}
\def\mX{{\bm{X}}}

\def\mZ{{\bm{Z}}}

\def\mSigma{{\bm{\Sigma}}}

\DeclareMathAlphabet{\mathsfit}{\encodingdefault}{\sfdefault}{m}{sl}
\SetMathAlphabet{\mathsfit}{bold}{\encodingdefault}{\sfdefault}{bx}{n}

\newcommand{\E}{\mathbb{E}}

\newcommand{\R}{\mathbb{R}}

\newcommand{\KL}{D_{\mathrm{KL}}}

\def\vfg{{\bm{f}_{G}}}
\def\vgg{{\bm{g}_{G}}}
\def\rightarrowd{{\xrightarrow{d}}}
\def\rightarrowp{{\xrightarrow{p}}}
\DeclareMathOperator*{\transit}{T}
\DeclareMathOperator*{\generator}{A}
\DeclareMathOperator*{\solver}{Solver}

\def\vggb{{\bm{\bar{g}}_{\bar{G}}}}
\def\vfgb{{\bm{\bar{f}}_{\bar{G}}}}
\def\mXb{{\bm{\bar{X}}}}

\def\pb{{\bar{p}}}
\def\mGb{{\bar{\bm{G}}}}
\def\mXd{{\bm{X}^\Delta}}
\def\mXdt{{\bm{X}^\Delta_t}}

\def\mXdb{{\bar{\bm{X}}^\Delta}}
\def\mXdtp{{\bm{X}^\Delta_{t+1}}}

\def\mZb{{\bar{\bm{Z}}}}
\def\mZt{{\tilde{\bm{Z}}}}

\def\Wt{{\tilde{\mW}}}

\def\tii{{t_{i+1}}}
\def\ti{{t_i}}
\def\PaG{{\bm{Pa}_\mG}}
\def\resolvent{{R}}

\newtheorem{theorem}{Theorem}[section]

\newtheorem{lemma}[theorem]{Lemma}
\newtheorem{definition}{Definition}
\newtheorem{assumption}{Assumption}

%% file: macros.tex
\newcommand{\vars}{\mX}
\newcommand{\var}{X}
\newcommand{\latentvars}{\mZ}
\newcommand{\latentvar}{Z}
\newcommand{\graph}{\mG}

\newcommand{\driftfn}{\bm{f}}
\newcommand{\wiener}{\bm{W}}
\newcommand{\wienersingle}{W}

\newcommand{\nnparams}{\theta}

\newcommand{\smpidx}{n}
\newcommand{\maxsmpidx}{N}
\newcommand{\varidx}{d}
\newcommand{\maxvaridx}{D}
\newcommand{\timects}{t}

\newcommand{\timeidx}{i}
\newcommand{\maxtimeidx}{I}

\newcommand{\ModelName}{\emph{SCOTCH}}

\newcommand{\intv}{\iota}

%% file: packages.tex
\usepackage{cleveref}
\usepackage{subcaption}
\usepackage{graphics}

%% file: Introduction.tex
\section{Introduction}
\label{sec: Introduction}

Time-series data is ubiquitous in the real world, often comprising a series of data points recorded at varying time intervals. Understanding the underlying structures between variables associated with temporal processes is of paramount importance for numerous real-world applications \citep{spirtes2000causation, berzuini2012causality, peters2017elements}. Although randomised experiments are considered the gold standard for unveiling such relationships, they are frequently hindered by factors such as cost and ethical concerns. Structure learning seeks to infer hidden structures from purely observational data, offering a powerful approach for a wide array of applications \citep{bellot2021neural, lowe2022amortized,runge2018causal,tank2021neural,pamfil2020dynotears,gong2022rhino}.

However, many existing structure learning methods for time series are discrete, assuming that the underlying temporal processes are discretized in time and requiring uniform sampling intervals. Consequently, these models face two limitations: (i) they may misrepresent the true underlying process when it is continuous, potentially leading to incorrect inferred relationships; and (ii) they struggle with handling irregular sampling intervals, which frequently arise in many fields \citep{trapnell2014dynamics, qiu2017reversed, qian2020learning} and climate science \citep{bracco2018advancing, raia2008causality}. 

To address these challenges, we introduce a novel framework, \textbf{S}tructure learning with \textbf{CO}ntinuous-\textbf{T}ime sto\textbf{CH}astic models
(\ModelName{}\footnote{\url{https://github.com/microsoft/causica/tree/main/research_experiments/scotch}}), which \cameraready{employs} \emph{stochastic differential equations} (SDEs) for structure learning in temporal processes.
\ModelName{} can \cameraready{naturally} handle irregularly sampled time series and accurately represent \cameraready{and learn} continuous\cameraready{-time} processes. We make the key contributions:
\begin{enumerate}
    \item We introduce a novel latent Stochastic Differential Equation (SDE) formulation for modelling 
    \cameraready{structure in} continuous-time observational time-series data. To effectively train our proposed model, which we denote as \ModelName, we adapt the variational inference framework proposed in \citep{li2020scalable, tzen2019neural} to approximate the posterior for both the underlying graph structure and the latent variables. \cameraready{We show that, in contrast to a prior approach using ordinary differential equations (ODEs) \citep{bellot2021neural}, our model is capable of accurately learning the underlying dynamics from trajectories exhibiting multimodality and with non-Gaussian distribution.}
    
    \item We provide a rigorous theoretical analysis to support our proposed methodology. Specifically, we prove that when SDEs are directly employed for modelling the observational process, the resulting SDEs are structurally identifiable under global Lipschitz and diagonal noise assumptions. We also prove our model maintains structural identifiability under certain conditions, even when adopting the latent formulation; and that variational inference, when integrated with the latent formulation, in the infinite data limit, can successfully recover the ground truth graph structure and mechanisms under specific assumptions.
    
    \item Empirically, 
    we conduct extensive experiments on both synthetic and real-world datasets \cameraready{showing that} \ModelName{} can improve upon existing methods on structure learning, including when the data is irregularly sampled.
    
\end{enumerate}

%% file: Preliminary.tex
\section{Preliminaries}
\label{sec: Preliminary} 

\cameraready{In the rest of this paper,} we use $\vars_t\in\R^D$ to denote the $D$-dimensional observation vector at time $t$, with $X_{t,d}$ representing the $d^{\text{th}}$ variable of the observation. A time series is a set of $I$ observations $\vars = \{\vars_{t_i}\}_{t=1}^I$, where $\{t_i\}_{i=1}^{I}$ are the observation times. 
In the case where we have multiple ($\maxsmpidx$) i.i.d. time series, we use $\vars^{(n)}$ to indicate the $n^{\text{th}}$ time series. 
\paragraph{Bayesian structure learning} In structure learning, the aim is to infer \cameraready{the underlying} graph representing the relationships between variables from data. \cameraready{In the Bayesian approach, given time series data $\{\vars^{(\smpidx)}\}_{\smpidx = 1}^{\maxsmpidx}$, we define a joint distribution over graphs and data given by:}
\begin{equation}
    p(\mG,\mX^{(1)},\ldots, \mX^{(\maxsmpidx)}) = p(\mG)\prod_{\smpidx=1}^{\maxsmpidx} p(\mX^{(\smpidx)}|\mG)
    \label{eq: joint distribution}
\end{equation}
where $p(\mG)$ is the graph prior and $p(\vars^{(\smpidx)}|\mG)$ is the likelihood term. The goal is then to compute the graph posterior $p(\mG|\cameraready{\mX^{(1)},\ldots \mX^{(\maxsmpidx)}})$. However, analytic computation is intractable in high dimensional settings. Therefore, variational inference \citep{zhang2018advances} and sampling methods \citep{welling2011bayesian, gong2018meta,annadani2023bayesdag} are commonly used for inference. 

\paragraph{Structural equation models (SEMs)} Given a time series $\vars$ and graph $\mG\in \{0,1\}^{D\times D}$, we can use SEMs to describe the structural relationships between variables: 
\begin{equation}
    X_{t,d} = f_{t,d}(\PaG^d(<t),\epsilon_{t,d})
    \label{eq: temporal SEM}
\end{equation}
where $\PaG^d(<t)$ specifies the lagged parents of $X_{t,d}$ at previous time and $\epsilon_{t,d}$ is the mutually independent noise. Such a model requires discrete time steps that are usually assumed to follow a regular sampling interval, i.e. $t_{i+1}-t_i$ is a constant for all $i = 1,\ldots,I - 1$.
Most existing models can be regarded as a special case of this framework.

\paragraph{It\^o diffusion}
A time-homogenous It\^o diffusion is a stochastic process $\mX_t$ 
and has the form:
\begin{equation}
    d\mX_t = \vf(\mX_t)dt +\vg(\mX_t)d\mW_t
    \label{eq: ito diffusion no graph}
\end{equation}
where $\vf:\R^D\rightarrow \R^D,\vg:\R^D \rightarrow \R^{D\times D}$ are time-homogeneous drift and diffusion functions, respectively, and $\mW_t$ is a Brownian motion under the measure $P$. It is known that under global Lipschitz guarantees (Assumption \ref{assump: Global Lipschitz}) this has a unique strong solution \citep{oksendal2003stochastic}.

\paragraph{Euler discretization and Euler SEM}
For most It\^o diffusions, the analytic solution $\mX_t$ is intractable, especially with non-linear drift and diffusion functions. Thus, we often seek to simulate the trajectory by discretization. One common \cameraready{discretization method} is the \emph{Euler-Maruyama} (EM) scheme. With a fixed step size $\Delta$, EM simulates the trajectory as 
\begin{equation}
    \mXdtp = \mXdt+\vf(\mXdt)\Delta + \vg(\mXdt)\eta_t
    \label{eq: EM updates no graph}
\end{equation}
where $\mXdt$ is the random variable induced by discretization and $\eta_t\sim \mathcal{N}(0,\Delta)$. Notice that \cref{eq: EM updates no graph} is a special case of \cref{eq: temporal SEM}. If we define the graph $\mG$ as the following: 
$\mX^\Delta_{t,i}\rightarrow \mX^\Delta_{t+1,j}$ in $\mG$ iff $\frac{\partial f_j(\mXdt)}{\partial X^\Delta_{t,i}}\neq 0$ or $\exists k, \frac{\partial g_{j,k}(\mXdt)}{\partial X^\Delta_{t,i}}\neq 0$; 
and assume 
$\vg$
only outputs a diagonal matrix, then the above EM induces a temporal SEM, called the \emph{Euler SEM} \citep{hansen2014causal}, which provides a useful analysis tool for continuous time processes.

%% file: Methodology.tex
\section{\ModelName{}: Bayesian Structure Learning for Continuous Time Series}
\label{sec: Methodology}
We consider a dynamical system in which there is both intrinsic stochasticity in the evolution of the state, as well as independent measurement noise that is present in the observed data. 
For example, in healthcare, the condition of a patient will progress with randomness rather than deterministically. 
\cameraready{Further,}
the measurement of patient condition will also be affected by the accuracy of the equipment, where the noise is independent to the intrinsic stochasticity. To account for this behaviour, we propose to use the latent SDE formulation \citep{li2020scalable, tzen2019neural}:
\begin{align}
    d\mZ_t&=\vf_{\nnparams}(\mZ_t)dt + \vg_{\nnparams}(\mZ_t)d\mW_t \text{ (latent process)}\nonumber\\
    \mX_t &= \mZ_t + \bm{\epsilon}_t \text{ (noisy observations)}
    \label{eq: CRhino datamodel}
\end{align}
where $\mZ_t\in\R^D$ is the latent variable representing the internal state of the dynamic system, $\mX_t\in \R^D$ describes the observational data with the same dimension, $\bm{\epsilon_t}$ is additive \cameraready{dimension-wise independent noise}, $\vf_{\nnparams}:\R^D \rightarrow \R^D$ is the drift function, $\vg_{\nnparams}:\R^D \rightarrow \R^{D \times D}$ is the diffusion function and $\mW_t$ is the Wiener process.  \revised{For \ModelName{}, we require the following two assumptions:}
\revised{
\begin{restatable}[Global Lipschitz]{assumption}{AssumpOne}
We assume that the drift and diffusion functions in \cref{eq: CRhino datamodel} satisfy the global Lipschitz constraints. Namely, we have
\begin{equation}
    \vert \vf_{\nnparams}(\vx) -\vf_{\nnparams}(\vy) \vert + \vert \vg_{\nnparams}(\vx)-\vg_{\nnparams}(\vy) \vert \leq C \vert \vx-\vy\vert
    \label{eq: global lipschitz}
\end{equation}
for some constant $C$, $\vx, \vy \in \R^D$ and $\vert\cdot\vert$ is the corresponding $L_2$ norm for vector-valued functions and matrix norm for matrix-valued functions.
\label{assump: Global Lipschitz}
\end{restatable}
\begin{restatable}[Diagonal diffusion]{assumption}{DiagDiff}
    We assume that the diffusion function $\vg_\theta$ outputs a non-zero diagonal matrix. That is, it can be simplified to a vector-valued function $\vg_\theta(\mX_t):\R^D\rightarrow \R^D$. 
\label{assump: diagonal diffusion}
\end{restatable} 
The former is a standard assumption required by most SDE literature to ensure the existence of a strong solution. The key distinction is the latter assumption of a nonzero diagonal diffusion function, $\vg_{\nnparams}$, rather than a full diffusion matrix, enabling structural identifiability as we show in the next section. Please refer to \cref{subapp: definitions} for more detailed explanations.} 
Now, in accordance with the graph defined in Euler SEMs (\cref{sec: Preliminary}), we define the \cameraready{\emph{signature graph}} $\mG$ as follows: edge $i\rightarrow j$ is present in $\mG$ iff~ $\exists t$ s.t.~either $\frac{\partial f_{j}(\mZ_t)}{\partial Z_{t,i}}\neq 0$ or $\frac{\partial g_{j}(\mZ_t)}{\partial Z_{t,i}}\neq 0$. Note that there is no requirement for the graph to be acyclic. Intuitively, the graph $\mG$ describes the structural dependence between variables.

\cameraready{We now instantiate our Bayesian structure learning scheme with prior and likelihood components:}
\paragraph{Prior over Graphs} Leveraging \cite{geffner2022deep, annadani2023bayesdag}, our graph prior is designed as:
\begin{equation}
    p(\mG)\propto\exp(-\lambda_s\Vert\mG\Vert_F^2)
    \label{eq: graph prior}
\end{equation}
where $\lambda_s$ is the graph sparsity coefficient, and $\Vert \cdot \Vert_F$ is the Frobenius norm. 

\paragraph{Prior process}
Since the latent process induces a distribution over latent trajectories \cameraready{$p_\theta(\mZ)$} before seeing any observations, we also call it the prior process. 
We propose to use neural networks for drift and diffusion functions $\vf_{\nnparams}:\R^D\times \{0,1\}^{D\times D} \rightarrow \R^D$, $\vg_{\nnparams}:\R^D\times \{0,1\}^{D\times D} \rightarrow \R^D$, which explicitly take the latent state and the graph as inputs. Though the signature graph is defined through the function derivatives, we explicitly use the graph $\mG$ as input to \cameraready{enforce the constraint}.
We will interchangeably use the notation $\vfg$ and $\vgg$ to denote $\vf_\theta(\cdot,\mG)$ and $\vg_\theta(\cdot,\mG)$. For the graph-dependent drift and diffusion, we leverage the design of \cite{geffner2022deep} and propose: 
\begin{equation}
    \vf_{G,d}(\mZ_t) = \zeta\left(\sum_{i=1}^D G_{i,d}l(Z_{t,i}, \ve_{i}), \ve_{d}\right)
    \label{eq: drift design}
\end{equation}
for both $\vfg$ and $\vgg$, where $\zeta$, $l$ are neural networks, and $\ve_i$ is a trainable node embedding for the $i^{\text{th}}$ node. The corresponding prior process is:
\begin{align}
    d\mZ_t&=\vf_\theta(\mZ_t,\mG)dt + \vg_\theta(\mZ_t,\mG)d\mW_t \text{ (prior process)}
    \label{eq: CRhino formulation}
\end{align}

\paragraph{Likelihood of time series} 
Given a time series $\mX = \{\vars_{t_i}\}_{i=1}^I$, the likelihood is defined as 
\cameraready{
\begin{equation}
    p(\{\vars_{t_i}\}_{i=1}^I\vert \{\mZ_{t_i}\}_{i=1}^I,\mG)=\prod_{i=1}^I\prod_{d=1}^D p_{\epsilon_d}(X_{t_i, d} - Z_{t_i,d})
    \label{eq: observation likelihood}
\end{equation}
}
where $p_{\epsilon_d}$ is the observational noise distribution for the $d^{\textnormal{th}}$ dimension.

\subsection{Variational Inference}
\label{subsec: variational inference}
Suppose that we are given multiple time series $\{\vars^{(n)}\}_{n=1}^N$ as observed data from the system. The goal is then to compute
the posterior over graph structures $p(\mG|\{\vars^{(n)}\}_{n=1}^N)$, which is intractable. Thus, we leverage variational inference to simultaneously approximate both the graph posterior, and a latent posterior process over $\mZ^{(\smpidx)}$ for every observed time series $\vars^{(n)}$.

Given $N$ i.i.d~time series $\{\mX^{(n)}\}_{n=1}^N$, we propose to use  a variational approximation $q_\phi(\mG)\approx p(\mG\vert \mX^{(1)},\ldots, \mX^{(N)})$. With the standard trick from variational inference, we have the following evidence lower bound (ELBO):
\begin{equation}
    \log p(\mX^{(1)},\ldots, \mX^{(N)}) \geq \E_{q_\phi(\mG)}\left[\sum_{n=1}^N\log p_\theta(\mX^{(n)}\vert \mG)\right] - \KL(q_\phi(\mG)\Vert p(\mG))
    \label{eq: ELBO 1}
\end{equation}
Unfortunately, the distribution $p_\theta(\mX^{(n)}\vert \mG)$ remains intractable due to the marginalization of the latent It\^o diffusion $\mZ^{(n)}$. Therefore, we leverage the variational framework proposed in
\cite{tzen2019neural,li2020scalable} to approximate the true posterior $p(\latentvars^{(\smpidx)} | \vars^{(\smpidx)}, \mG)$. For each $n=1,\ldots,N$, the variational posterior $q_{\psi}(\tilde{\latentvars}^{(\smpidx)} | \vars^{(\smpidx)}, \mG)$ is defined as the solution to the following:
\begin{align}
    \mZt_{t, 0}^{(n)} \sim \mathcal{N}(\vmu_\psi(\graph, \mX^{(n)}), \mSigma_\psi(\graph, \mX^{(n)})) \;\;\;\text{(posterior initial state)}\nonumber\\ 
    d\mZt_t^{(n)}= \vh_{\psi}(\mZt_t^{(n)},t; \graph, \mX^{(n)})dt+\vgg(\mZt_t^{(n)})d\mW_t \;\;\;\text{(posterior process)}
    \label{eq: posterior process}
\end{align}
For the initial latent state, $\vmu_\psi, \mSigma_\psi$ are posterior mean and covariance functions implemented as neural networks. For the SDE, we use the same diffusion function $\vgg$ for both the prior and posterior processes, but train a separate neural drift function $\vh_{\psi}$ for the posterior, which takes a time series $\mX^{(n)}$ as input. The posterior drift function differs from the prior in two key ways. Firstly, the posterior drift function depends on time; this is necessary as conditioning on the observed data creates this dependence even when the prior process is time-homogenous. Secondly, while $\vh_{\psi}$ takes the graph $\mG$ as an input, the function design is not constrained to have a matching signature graph like $\vfg$. More details on the implementation of $\vh_{\psi}, \vmu_\psi, \mSigma_\psi$ can be found in Appendix \ref{app: model architecture}.

Assume for each time series $\mX^{(n)}$, we have observation times $t_i$ for $i=1,\ldots, I$ within the time range $[0,T]$, then, we have the following evidence lower bound for $\log p(\mX^{(n)}\vert\mG)$ \citep{li2020scalable}:
\begin{equation}
    \log p(\mX^{(n)}\vert\mG)\geq \E_{q_{\psi}}\left[\sum_{i=1}^I \log p(\mX^{(n)}_{t_i}\vert \mZt^{(n)}_{t_i},\mG)-\int_0^T \Vert \vu^{(n)}(\mZt^{(n)}_t)\Vert^2dt\right]
    \label{eq: ELBO 2}
\end{equation}
where $\mZt^{(n)}_t$ is the posterior process modelled by \cref{eq: posterior process} and $\vu^{(n)}(\mZt^{(n)}_t)$ is given by:
\begin{equation}
    \vu^{(n)}(\mZt^{(n)}_t)=\vgg(\mZt_t^{(n)})^{-1}(\vh_{\psi}(\mZt_t^{(n)}, t; \mG, \mX^{(n)})-\vf_{G}(\mZt_t^{(n)}))
    \label{eq: u}
\end{equation}

By combining \cref{eq: ELBO 1} and \cref{eq: ELBO 2}, we derive an overall ELBO:
\begin{align}
    \log p_\theta(\mX^{(1)},\ldots, \mX^{(N)}) \geq& \E_{q_\phi}\left[\sum_{n=1}^N \E_{q_\psi}\left[\sum_{i=1}^I \log p(\mX_{t_i}^{(n)}\vert \mZt_{t_i}^{(n)},\mG)-\int_0^T \Vert\vu^{(n)}(\mZt_t^{(n)})\Vert^2dt\right]\right]\nonumber\\
    &-\KL(q_\phi(\mG)\Vert p(\mG))
    \label{eq: Overall ELBO}
\end{align}
In practice, we approximate the ELBO (and its gradients) using a Monte-Carlo approximation. The inner expectation can be approximated by simulating from an augmented version of \cref{eq: posterior process} where an extra variable $L$ is added with drift $\frac{1}{2}\vert \vu^{(n)}(\mZt^{(n)}_t)\vert^2$ and diffusion zero \citep{li2020scalable}. \Cref{alg: CRhino training} summarizes the training algorithm of \ModelName{}.

\begin{algorithm}[tb]
\small
\caption{\ModelName{} training}
\label{alg: CRhino training}

\begin{algorithmic}
\STATE {\bfseries Input:} i.i.d time series $\{\mX^{(n)}\}_{n=1}^N$; drift functions $\vfg$, $\vh_{\psi}$, diffusion function $\vgg$, SDE solver $\solver$, initial condition $\mZt_0^{(n)}$, training iterations L
\FOR{$l=1,\ldots,L$}
    \STATE Sample time series mini-batch $\{\mX^{(n)}\}_{n=1}^S$ with batch size $S$.
    \FOR{$n=1,\ldots,S$}
    \STATE Draw graph $\mG \sim q_\phi(\mG)$
    \STATE Draw initial latent state $\tilde{\latentvars}^{(n)}_{0} \sim \mathcal{N}(\vmu_\psi(\graph, \mX^{(n)}), \mSigma_\psi(\graph, \mX^{(n)}))$
    \STATE Solve (sample from) the posterior process $(\mZt^{(n)}, L) = \solver((\mZt^{(n)}_0, 0),\vfg, \vh_{\psi} ,\vgg)$
    \ENDFOR
     
    \STATE Maximize ELBO \cref{eq: Overall ELBO} w.r.t.~$\phi, \psi, \theta$
\ENDFOR
\end{algorithmic}
\end{algorithm}

\subsection{\cameraready{Stochasticity and Continuous-Time Modeling}}
\label{subsec: comparison to Bellot}

\paragraph{\cameraready{Stochasticity}} \citet{bellot2021neural} proposed a structure learning method, called NGM, to learn from single time series generated by SDEs. NGM uses a neural ODE to model the mean process $\vf_\theta$, and extracts graphical structure from the first layer of $\vf_\theta$. However, NGM assumes that the observed single series $\mX$ follows a multivariate Gaussian distribution, which only holds for linear SDEs. 
If this assumption is violated, optimizing their proposed squared loss cannot recover the underlying system. 
\ModelName{} does not have this limitation and can handle more flexible state-dependent drifts and diffusions. 
Another drawback of NGM is its inability to handle multiple time series ($\maxsmpidx > 1$). Learning from multiple series is important when dealing with SDEs with multimodal behaviour. We propose a simple bimodal 1-D failure case: $d\var = \var d\timects + 0.01 d\wienersingle_{\timects}, \var_0 = 0$, with the signature graph containing a self-loop. Figure \ref{fig:failure_case} shows the bimodal trajectories (upwards and downwards) sampled from the SDE. The optimal ODE mean process in this case is the constant $\vf_\theta=0$ with an empty graph, as confirmed by the learned mean process of NGM (black line in \cref{fig:failure_case_bellot}). In contrast, \ModelName{} can learn the underlying SDE and simulate the correct trajectories (\cref{fig:failure_case_crhino}).

\begin{figure}[tb]
    \centering
    \begin{subfigure}{0.32\linewidth}
    \centering
        \includegraphics[scale=0.35]{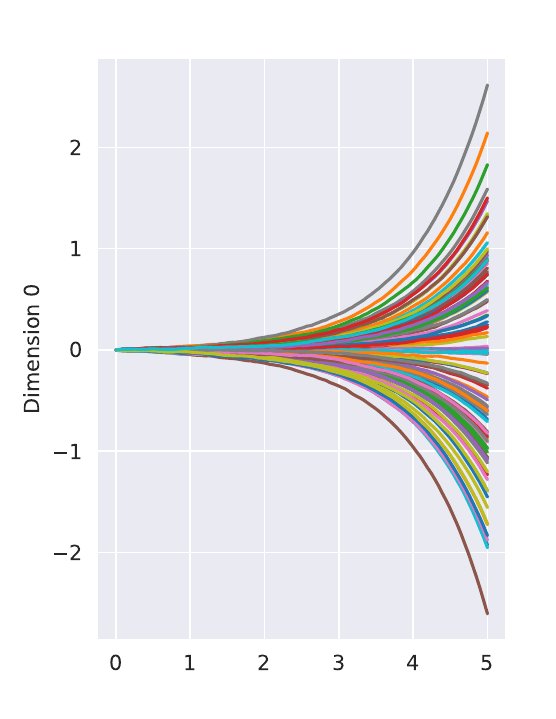}
        \caption{Data}
        \label{fig:failure_case_data}
    \end{subfigure}
    \begin{subfigure}{0.32\linewidth}
    \centering
        \includegraphics[scale=0.35]{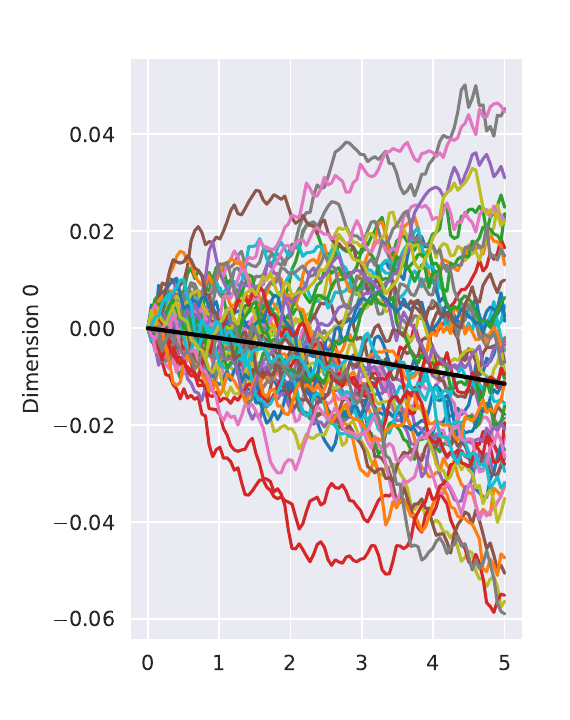}
        \caption{NGM}
        \label{fig:failure_case_bellot}
    \end{subfigure}
    \begin{subfigure}{0.32\linewidth}
    \centering
        \includegraphics[scale=0.35]{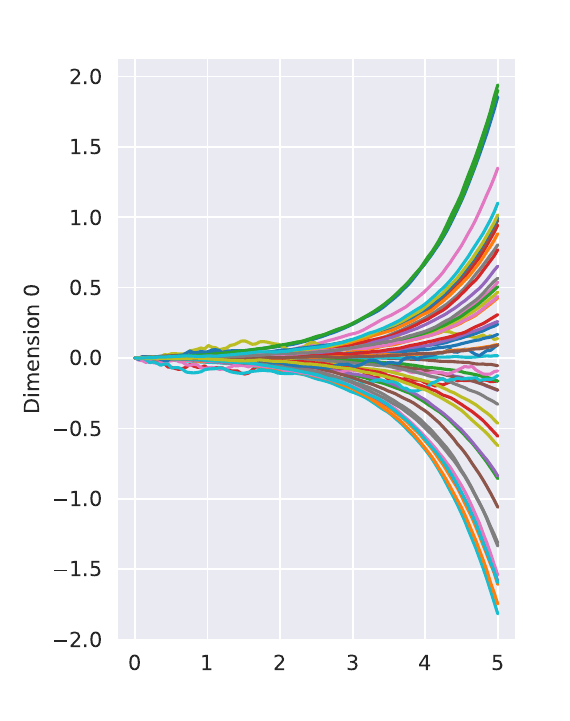}
        \caption{\ModelName{}}
        \label{fig:failure_case_crhino}
    \end{subfigure}
    \caption{Comparison between NGM and \ModelName{} for simple SDE (note vertical axis scale)}
    \label{fig:failure_case}
\end{figure}
\paragraph{\cameraready{Discrete vs Continuous-Time}} \cite{gong2022rhino} proposed a flexible discretised temporal SEM, called Rhino, that is capable of modelling (1) lagged parents; (2) instantaneous effect; and (3) history dependent noise. Rhino's SEM is given by $X_{t,d}=f_d(\PaG^d(<t), \PaG^d(t))+g_d(\PaG^d(<t))\epsilon_{t,d}$. We can clearly see its similarity to \ModelName{}. If $f_d$ has a residual structure as $f_d(\cdot) = X_{t,d} + r_d(\cdot)\Delta$ and \highlight{we assume no instantaneous effect ($\PaG^d(t)$ is empty)}, Rhino SEM is equivalent to the Euler SEM of the latent process (\cref{eq: CRhino formulation}) with drift $\vr$, step size $\Delta$ and diagonal diffusion $\vg$. Thus, similar to the relation of ResNet \citep{he2016deep} to NeuralODE \citep{chen2018neural}, \ModelName{} \cameraready{can be interpreted} as the continuous-time analog of Rhino.

%% file: Theory_main.tex
\section{Theoretical considerations of \ModelName{}}
\label{sec: theory main text}

In this section, we aim to answer three important theoretical questions regarding the It\^o diffusion proposed in \cref{sec: Methodology}. For notational simplicity, we consider the single time series setting. First, we examine when a general It\^o diffusion is structurally identifiable. 
Secondly, we consider structural identifiability in the latent formulation of  \cref{eq: CRhino datamodel}. Finally, we consider whether optimising ELBO (\cref{eq: Overall ELBO}) can recover the true graph and mechanism if we have infinite observations of a single time series within a fixed time range $[0,T]$. All detailed proofs, definitions, and assumptions can be found in \cref{app: theory}.

\subsection{Structure identifiability}
\label{subsec: structure identifiability}
Suppose that the observational process is given as an It\^o diffusion:
\begin{equation}
    d\mX_t = \vfg(\mX_t)dt + \vgg(\mX_t)d\mW_t
    \label{eq: Observational process}
\end{equation}
Then, we might ask what are sufficient conditions for the model to be structurally identifiable? 
That is, there does not exist $\mG'\neq \mG$ that can induce the same observational distribution. 

\begin{restatable}[Structure identifiability of the observational process]{theorem}{IdenObser}
    Given \cref{eq: Observational process}, let us define another process with 
$\bar{\mX}_t$, $\mG\neq \bar{\mG}$, $\vfgb$, $\vggb$ and $\bar{\mW}_t$. Then, under Assumptions \ref{assump: Global Lipschitz}-\ref{assump: diagonal diffusion}, and with the same initial condition $\mX(0)=\mXb(0)=\vx_0$, the solutions $\mX_t$ and $\mXb_t$ will have different distributions. 
\label{thm: structure identifiability observation}
\end{restatable}

Next, we show that structural identifiability is preserved, under certain conditions, even in the latent formulation where the SDE solution is not directly observed.
\begin{restatable}[Structural identifiability with latent formulation]{theorem}{IdenLatent}

Consider the distributions $p, \pb$ defined by the latent model in  \cref{eq: CRhino datamodel} with $(\mG, \mZ, \mX, \vfg, \vgg), (\mGb, \mZb, \mXb, \vfgb, \vggb)$ respectively, where $\mG \neq \mGb$. Further, let $t_1,\ldots,t_I$ be the observation times. Then, under Assumptions \ref{assump: Global Lipschitz} and \ref{assump: diagonal diffusion}:
    \begin{enumerate}
        \item if $t_{i+1}-t_{i}=\Delta$ for all $i\in 1, ..., I-1$, then $p^\Delta(\mX_{t_1},\ldots, \mX_{t_I})\neq \pb^\Delta(\mXb_{t_1},\ldots, \mXb_{t_I})$, where $p^\Delta$ is the density generated by the Euler discretized \cref{eq: CRhino formulation} for $\mZ_t$;
        \item if we have a fixed time range $[0,T]$, then the path probability $p(\mX_{t_1},\ldots, \mX_{t_I})\neq \pb(\mXb_{t_1},\ldots, \mXb_{t_I})$ under the limit of infinite data ($I\rightarrow \infty$).
    \end{enumerate}
    \label{thm: identifiability of latent SDE}
\end{restatable}

\subsection{Consistency}
\label{subsec: Consistency}
Building upon the structural identifiability, we can prove the consistency of the variational formulation. Namely, in the infinite data limit, one can recover the ground truth graph and mechanism by maximizing ELBO with a sufficiently expressive posterior process and a correctly specified model. 
\begin{restatable}[Consistency of variational formulation]{theorem}{Consistency}
Suppose Assumptions \ref{assump: Global Lipschitz}-\ref{assump: expressive posterior} are satisfied for the latent formulation (\cref{eq: CRhino datamodel}). Then, for a fixed observation time range $[0,T]$, as the number of observations $I\rightarrow \infty$, when ELBO (\cref{eq: Overall ELBO}) is maximised, $q_\phi(\mG)=\delta(\mG^*)$, where $\mG^*$ is the ground truth graph, and the latent formulation recovers the underlying ground truth mechanism.  
\label{thm: Consistency}
\end{restatable}

%% file: Related_work.tex
\section{Related work}
\label{sec: related work}
\paragraph{Discrete time causal models}
The majority of the existing approaches are inherently discrete in time. \cite{assaad2022survey} provides a comprehensive overview. There are three types of discovery methods: (1) Granger causality; (2) structure equation model (SEM); and (3) constraint-based methods. Granger causality assumes that no instantaneous effects are present and the causal direction cannot flow backward in time. \cite{wu2020discovering, shojaie2010discovering, siggiridou2015granger, amornbunchornvej2019variable} leverage the vector-autoregressive model to predict future observations. \cite{lowe2022amortized, tank2021neural, bussmann2021neural, dang2019seq2graph,xu2019scalable, khanna2019economy} utilise deep neural networks for prediction. Recently, \cite{cheng2023cuts} introduced a deep-learning based Granger causality that can handle irregularly sampled data, treating it as a missing data problem and proposing a joint framework for data imputation and graph fitting. SEM based approaches assume an explicit causal model associated to the temporal process. \cite{hyvarinen2010estimation} leverages the identifiability of additive noise models \citep{hoyer2008nonlinear} to build a linear auto-regressive SEM with non-Gaussian noise. \cite{pamfil2020dynotears} utilises the NOTEARS framework \citep{zheng2018dags} to continuously relax the DAG constraints for fully differentiable structure learning. The recently proposed \cite{gong2022rhino} extended the prior DECI \cite{geffner2022deep} framework to handle time series data and is capable of modelling instantaneous effect and history-dependent noise. Constraint-based approaches use conditional independence tests to determine the causal structures. \cite{runge2019detecting} combines the PC \citep{spirtes2000causation} and momentary conditional independence tests for the lagged parents. PCMCI+ \citep{runge2020discovering} can additionally detect the instantaneous effect. LPCMCI \citep{reiser2022causal} can further handle latent confounders. CD-NOD \citep{zhang2017causal} is designed to handle non-stationary heterogeneous time series data. However, all constraint-based approaches can only identify the graph up to Markov equivalence class without the functional relationship between variables. 

\paragraph{Continuous time causal models}
In terms of using differential equations to model the continuous temporal process, \cite{hansen2014causal} proposed using stochastic differential equations to describe the temporal causal system. They proved identifiability with respect to the intervention distributions, but did not show how to learn a corresponding SDE. Penalised regression has been explored for linear models, where parameter consistency has been established \citep{ramsay2007parameter, chen2017network, wu2014sparse}. Recently, NGM \citep{bellot2021neural} uses ODEs to model the temporal process with both identifiability and consistency results. As discussed in previous sections, \ModelName{} is based on SDEs rather than ODEs, and can model the intrinsic stochasticity within the causal system, whereas NGM assumes deterministic state transitions. 

%% file: Experiments.tex
\section{Experiments}
\label{sec: experiments}

\paragraph{Baselines and Metrics} We benchmark our method against a representative sample of baselines: (i) VARLiNGaM \citep{hyvarinen2010estimation}, a linear SEM based approach; (ii) PCMCI+ \citep{runge2018causal, runge2020discovering}, a constraint-based method for time series; (iii) CUTS, a Granger causality approach which can handle irregular time series;
(iv) Rhino \citep{gong2022rhino}, a non-linear SEM based approach with history-dependent noise and instantaneous effects; and (v) NGM \citep{bellot2021neural}, a continuous-time ODE based structure learner.
Since most methods require a threshold to determine the graph, we use the threshold-free \emph{area under the ROC curve} (AUROC) as the performance metric. In appendix \ref{app: Experiments}, we also report F1 score, true positive rate (TPR) and false discovery rate (FDR).

\paragraph{Setup} Both the synthetic datasets (Lorenz-96, Glycolysis) and real-world datasets (DREAM3, Netsim) consist of multiple time series. However, it is not trivial to modify NGM and CUTS to support multiple time series. For fair comparison, we use the concatenation of multiple time series, which we found empirically to improve performance. We also mimic irregularly sampled data by randomly dropping observations, which VARLiNGaM, PCMCI, and Rhino cannot handle; in these cases, for these methods we impute the missing data using zero-order hold (ZOH). 
\addition{For \ModelName{}, we use pathwise gradient estimators with Euler discretization for solving the SDE (see \cref{subapp: choice sde solver} for discussion on this choice).} 
Further experimental details can be found in Appendices \ref{app: model architecture}, \ref{app: baselines}, \ref{app: Experiments}.
\subsection{Synthetic experiments: Lorenz and Glycolysis}
\label{subsec: synthetic experiments}
First, we evaluate \ModelName{} on synthetic benchmarks including the Lorenz-96 \citep{lorenz1996predictability} and Glycolysis \citep{daniels2015efficient} datasets, which model continuous-time dynamical systems. The Lorenz model is a well-known example of chaotic systems observed in biology \citep{goldberger1987applications,heltberg2019chaotic}. To mimic irregularly sampled data, we follow the setup of \cite{cheng2023cuts} and randomly drop some observations with missing probability $p$. 
We also simulate another dataset from a biological model, which describes metabolic iterations that break down glucose in cells. This is called \emph{Glycolysis}, consisting of an SDE with $7$ variables. As a preprocessing step, we standardised this dataset to avoid large differences in variable scales.
Both datasets consist of $\maxsmpidx = 10$ time series with sequence length $\maxtimeidx = 100$ (before random drops), and have dimensionality $10$ and $7$, respectively. Note that we choose a large data sampling interval, as we want to test settings where observations are fairly sparse and the difficulty of correctly modelling continuous-time dynamics increases. The above data setup is different from \cite{bellot2021neural, cheng2023cuts} where they use a single series with $I = 1000$ observations, which is more informative compared to our sparse setting.
Refer to \cref{subapp: synthetic lorenz} and \cref{subapp: synthetic glycolysis} for details. 

The left two columns in \cref{tab: synthetic results} compare the AUROC of \ModelName{} to baselines for Lorenz. We can see that \ModelName{} can effectively handle the irregularly sampled data compared to other baselines. Compared to NGM and CUTS, we can achieve much better results with small missingness and performs competitively with larger missingness.
Rhino, VARLiNGaM and PCMCI+ perform poorly in comparison as they assume regularly sampled observations and are discrete in nature.

From the right column in \cref{tab: synthetic results}, \ModelName{} outperforms the baselines by a large margin on Glycolysis. In particular, compared to the ODE-based NGM, \ModelName{} clearly demonstrates the advantage of the proposed SDE framework in multiple time series settings. As we may have anticipated from the discussion in \cref{subsec: comparison to Bellot}, NGM can produce an incorrect model when multiple time series are sampled from a given SDE system. Another interesting observation is that \ModelName{} is more robust when encountering data with different scales compared to NGM (refer to \cref{subsubapp: glycolysis additional}). This robustness is due to the stochastic nature of SDE compared to the deterministic ODE, where ODE can easily overshoot with less stable training behaviour. We can also see that \ModelName{} has a significant advantage over both CUTS and Rhino, which do not model continuous-time dynamics.
\begin{table}[!h]
\centering
\scalebox{0.95}{
\begin{tabular}{l|ll|l}
\hline
              & \multicolumn{2}{c|}{\textbf{Lorenz-96}}                            & \multicolumn{1}{c}{\textbf{Glycolysis}} \\ 
              & \multicolumn{1}{c}{$p=0.3$} & \multicolumn{1}{c|}{$p=0.6$} & \multicolumn{1}{c}{Full}       \\ \hline
VARLiNGaM &0.5102$\pm$0.025 &0.4876$\pm$0.032 & 0.5082$\pm$0.009 \\
PCMCI+        &   0.4990$\pm$0.008                        &        0.4952$\pm$0.021                    &           0.4607$\pm$0.031                     \\
NGM           &  0.6788$\pm$0.009                       &     0.6329$\pm$0.008                   &     0.5953$\pm$0.018                          \\
CUTS          & 0.6042$\pm$0.015   
&  0.6418$\pm$0.012
&              0.580$\pm$0.007                  \\
Rhino         &      0.5714$\pm$0.026                     &           0.5123$\pm$0.025                 &        0.520$\pm$0.015\\
\ModelName{} (ours) &  \textbf{0.7279$\pm$0.017}                         &   \textbf{0.6453$\pm$0.014}                         &      \textbf{0.7113$\pm$0.012 }                         \\ \hline
\end{tabular}
}
\caption{AUROC of synthetic datasets from \ModelName{} and baselines. $p$ represents missing probability, and \emph{Full} means complete data without missingness. Each number is the average over $5$ runs. }
\label{tab: synthetic results}
\end{table}
\subsection{Dream3}
\label{subsec: Dream3}
We also evaluate \ModelName{} performance on the DREAM3 datasets \citep{prill2010towards,marbach2009generating}, which have been adopted for assessing the performance of structure learning \citep{tank2021neural, pamfil2020dynotears, gong2022rhino}. These datasets contain \emph{in silico} measurement of gene expression levels for $5$ different structures. Each dataset corresponds to a particular gene expression network, and contains $\maxsmpidx = 46$ time series of 100 dimensional variables, with $\maxtimeidx=21$ per series. The goal is to infer the underlying structures from each dataset. Following the same setup as \citep{gong2022rhino,khanna2019economy}, we ignore all the self-connections by setting the edge probability to $0$, and use AUROC as the performance metric. 
\Cref{subapp: dream3} details the experiment setup, selected hyperparameters, and additional plots. We do not include VARLiNGaM since it cannot support \cameraready{time} series where the dimensionality ($100$) is greater than the length ($21$). Also, due to the time series length, we decide not to test with irregularly sampled data. We use the reported numbers for Rhino and PCMCI+ in \citet{gong2022rhino} as the experimental setup is identical. For CUTS, we failed to reproduce the reported number in their paper, but we cite it for a fair comparison. 

\Cref{tab: Dream3 results} shows the AUROC performances of \ModelName{} and baselines. 
We can clearly observe that \ModelName{} outperforms the other baselines with a large margin. This indicates the advantage of the SDE formulation compared to ODEs and discretized temporal models, even when we have complete and regularly sampled data. A more interesting observation is to compare Rhino with \ModelName{}. As discussed before, as \ModelName{} is the continuous version of Rhino, the advantage comes from the continuous formulation and the corresponding training objective \cref{eq: Overall ELBO}. 

\begin{table}[]
\centering
\resizebox{\columnwidth}{!}{%
\begin{tabular}{@{}lllllll@{}}
\toprule
       & \textbf{EColi1} & \textbf{Ecoli2} &\textbf{Yeast1}&\textbf{Yeast2}&\textbf{Yeast3} & Mean \\ \midrule
PCMCI+ & 0.530$\pm$0.002 & 0.519$\pm$0.002 & 0.530 $\pm$0.003 & 0.510$\pm$0.001&  0.512 $\pm$ 0& 0.520$\pm$0.004 \\
NGM & 0.611$\pm$0.002 & 0.595$\pm$0.005 & 0.597$\pm$0.006 & 0.563$\pm$0.006 & 0.535$\pm$0.004 & 0.580$\pm$0.007
\\  
CUTS & 0.543$\pm$0.003 & 0.555$\pm$0.005 & 0.545$\pm$0.003 & 0.518$\pm$0.007 & 0.511$\pm$0.002 & 0.534$\pm$0.008 (0.591) \\
Rhino & 0.685$\pm$0.003 & 0.680$\pm$0.007 & 0.664$\pm$0.006 &0.585$\pm$0.004 &0.567$\pm$0.003 & 0.636$\pm$0.022\\

\ModelName{} (ours) & \textbf{0.752$\pm$0.008}          &   \textbf{0.705$\pm$0.003}  & \textbf{0.712$\pm$0.003}  & \textbf{0.622 $\pm$ 0.004} &  \textbf{0.594$\pm$ 0.001} &\textbf{0.677$\pm$ 0.026}   \\ \bottomrule
\end{tabular}
}
\caption{AUROC for \ModelName{} on DREAM3 100-dimensional datasets. Results are obtained by averaging over 5 runs. We cite the reported CUTS performance in parentheses.
}
\label{tab: Dream3 results}
\end{table}

\subsection{Netsim}
\label{subsec: Netsim}
Netsim consists of \emph{blood oxygenation level dependent} imaging data. Following the same setup as \cite{gong2022rhino}, we use subjects 2-6 to form the dataset, which consists of 5 time series. Each contains $15$ dimensional observations with $I=200$. The goal is to infer the underlying connectivity between different brain regions. Unlike Dream3, we include the self-connection edge for all methods. To evaluate the performance under irregularly sampled data, we follow the same setup as in the Lorenz and \citet{cheng2023cuts} to randomly drop observations with missing probability. \highlight{Since it is very important to model instantaneous effects in Netsim \citep{gong2022rhino}, which \ModelName{} cannot handle, we replace Rhino with Rhino+NoInst and PCMCI+ with PCMCI for fair comparison.} 

\Cref{tab: Netsim results} shows the performance comparisons. We  observe that \ModelName{} significantly outperforms the other baselines and performs on par with Rhino+NoInst, which demonstrates its robustness towards smaller datasets and balance between true and false positive rates. Again, this confirms the modelling power of our approach compared to NGM and other baselines. Interestingly, Rhino-based approaches perform particularly well on the Netsim dataset.
We suspect that the underlying generation mechanism can be better modelled with a discretized as opposed to continuous system.

\begin{table}[]
\centering
\begin{tabular}{@{}llll@{}}
\toprule
         & Full & $p=0.1$ &$p=0.2$      \\ \midrule
VARLiNGaM & 0.84$\pm$0 &0.723$\pm$0.001 &0.719$\pm$0.003 \\
PCMCI & 0.83$\pm$0 &0.81$\pm$0.001 &0.79$\pm$0.006\\
NGM & 0.89 $\pm$ 0.009 & 0.86 $\pm$ 0.009 & 0.85 $\pm$0.007 \\
CUTS & 0.89 $\pm$ 0.010 & 0.87 $\pm$ 0.008 & 0.87 $\pm$0.011 \\
Rhino+NoInst & \textbf{0.95 $\pm$0.001}  & \textbf{0.93$\pm$ 0.005} & \textbf{0.90$\pm$0.012}\\
\ModelName{} (ours) & \textbf{0.95$\pm$ 0.006}  & 0.91$\pm$0.007 & 0.89$\pm$0.007              \\ \bottomrule
Rhino& \textbf{0.99$\pm$0.001} & \textbf{0.98$\pm$0.004}& \textbf{0.97$\pm$0.003}\\ \bottomrule
\end{tabular}
\caption{AUROC on Netsim dataset (subjects 2-6). Results are obtained by averaging over 5 runs.}

\label{tab: Netsim results}
\end{table}

%% file: Conclusion.tex
\section{Conclusion}
\label{sec:conclusion}
We propose \ModelName{}, a flexible continuous-time temporal structure learning method based on latent It\^o diffusion. We leverage the variational inference framework to infer the posterior over latent states and the graph. Theoretically, we validate our approach by proving the structural identifiability of the It\^o diffusion and latent formulation, and the consistency of the proposed variational framework. Empirically, we extensively evaluated our approach using synthetic and semi-synthetic datasets, where \ModelName{} outperforms the baselines in both regularly and irregularly sampled data. There are three limitations that require further investigation. The first one is the inability to handle instantaneous effects, which can arise due to data aggregation. Another computational drawback is it scales linearly with the series length. This could be potentially fixed by incorporating an encoder network to infer latent states at arbitrary time points. \addition{Last but not least, the current formulation of \ModelName{} cannot handle non-stationary systems due to the homogeneous drift and diffusion function. However, direct incorporation of time embeddings may break the theoretical guarantees without additional assumptions. Therefore, new theories and methodologies may be needed to tackle such a scenario.} We leave these challenges for future work.

%% file: Acknowledgements.tex
\section*{Acknowledgements}

We thank the members of the Causica team at Microsoft Research for helpful discussions. We thank Colleen Tyler, Maria Defante, and Lisa Parks for conversations on real-world use cases that inspired this work. This work was done in part while Benjie Wang was visiting the Simons Institute for the Theory of Computing.

%% file: Theory.tex
\section{Identifiability of stochastic differential equations}
\label{app: theory}

\subsection{Definitions and assumptions}
\label{subapp: definitions}
In this part, we will introduce some basic definitions
and assumptions 
required for the theory. 

First, let us restate assumption \ref{assump: Global Lipschitz}.
\AssumpOne*
This assumption regularizes the It\^o diffusion to have a unique strong solution $\mX_t$ to \cref{eq: ito diffusion no graph}, which is a standard assumption in the SDE literature. In addition, this diffusion satisfies the Feller continuous property, and its solution is a Feller process (Lemma 8.1.4 in \cite{oksendal2003stochastic}). 

\begin{definition}[Feller process and semi-group]
    A continuous time-homogeneous Markov family $\mX_t$ is a Feller process when, for all $\vx \in \R^D$, we have 
    $\forall t, \vy\rightarrow \vx \Rightarrow \mX_{y,t}\rightarrowd \mX_{x,t}$ and $t\rightarrow 0 \Rightarrow \mX_{x,t} \rightarrowp \vx$ where $\rightarrowd$, $\rightarrowp$ means convergence in distribution and in probability, respectively, and $\mX_{y,t}$ means the solution with $y$ as the initial condition.  A semigroup of linear, positive, conservative contraction operators $\transit_t$ is a Feller semigroup if, for every $\vf \in C_0, \vx \in \R^D$, we have $\transit_t\vf \in C_0$ and $\lim_{t\rightarrow 0}\transit_t\vf(\vx) = \vf(\vx)$, where $C_0$ is the space of continuous functions vanishing at infinity.  
    \label{def: feller process}
\end{definition}
Basically, the transition operator of a Feller process is a Feller semigroup. The reason we care about the Feller process is its nice properties related to its infinitesimal generators. In a nutshell, the distributional properties of the Feller process can be uniquely characterised by its generators. 

\begin{definition}[Infinitesimal generator]
    For a Feller process $\mX_t$ with a Feller semigroup $\transit_t$, we define the generator $\generator$ by 
    \begin{equation}
        \generator f = \lim_{t\downarrow 0}\frac{\transit_t f - f}{t}\;\;\;\; \text{for any } f \in D(A)
        \label{eq: generator}
    \end{equation}
    where $D(A)$ is the domain of the generator, defined as the function space where the above limit exists. 
    \label{def: generator}
\end{definition}

Next, let us restate assumption \ref{assump: diagonal diffusion}.
\DiagDiff*
This is a key assumption for structure identifiability. For a general matrix diffusion function, it is easy to come up with unidentifiable examples (see Example 5.5 in \citep{hansen2014causal}). For example, in a driftless process, the distribution of $\mX_t$ will depend on $\vgg\vgg^T$, where it can have multiple factorizations that correspond to different graphs.

\subsection{Structure identifiability for observational process}
\label{subapp: structure identifiability observational process}
Now, let us re-state \cref{thm: structure identifiability observation}:
\IdenObser*

To prove this theorem, we begin by establishing the corresponding result for the discretized Euler SEMs, and then build the connection to the It\^o diffusion through the infinitesimal generator.

\begin{lemma}[Identifiability of Euler SEM]
    Assuming assumption \ref{assump: diagonal diffusion} is satisfied with nonzero diagonal diffusion functions. For a Euler SEM defined as 
    \begin{equation}
        \mX^\Delta_{t+1} = \mX^\Delta_{t}+\vfg(\mX^\Delta_t)\Delta+\vgg(\mX^\Delta_t)\bm{\eta}_t,\;\;\;\; \bm{\eta}_t\sim \mathcal{N}(0, \Delta\mI),
    \end{equation}
    if we have another Euler SEM defined as 
    \begin{equation}
        \bar{\mX}^\Delta_{t+1} = \mXb^\Delta_{t}+\vfgb(\mXb^\Delta_t)\Delta+\vggb(\mXb^\Delta_t)\bm{\bar{\eta}}_t,\;\;\;\; \bm{\bar{\eta}}_t\sim \mathcal{N}(0, \Delta\mI).
    \end{equation}
Then their corresponding transition density $p(\mX^\Delta_{t+1}\vert \mX^\Delta_t=\va) = \pb(\mXb^\Delta_{t+1}\vert \mXb^\Delta_t=\va)$ for all $\va\in\R^D$ iff. $\mG=\bar{\mG}$, $\vfg=\bar{\vf_G}$ and $\vert\vgg\vert=\vert\bar{\vg}_G\vert$. 
\label{lem: Identifiability of Euler SEM}
\end{lemma}
\begin{proof}
    If we have $\mG=\bar{\mG}$, $\vf=\bar{\vf}$ and $\vert\vg\vert=\vert\bar{\vg}\vert$, then it is trivial that their transition densities are the same since they define the same Euler SEM update equations (up to the sign of the diffusion term) with given initial conditions. 

    On the other hand, we know 
    \begin{align*}
        p(\mX^\Delta_{t+1}\vert \mX^\Delta_t=\va) &= \mathcal{N}(\vfg(\va)\Delta+\va, \vgg^2(\va)\Delta)\\
        \pb(\mXb^\Delta_{t+1}\vert \mXb^\Delta_t=\va) &= \mathcal{N}(\vfgb(\va)\Delta+\va, \vggb^2(\va)\Delta)
    \end{align*}
    Thus, if two conditional distributions match, we have 
    \begin{equation}
        \vfg(\va)\Delta = \vfgb(\va)\Delta\;\;\;\; \vgg^2(\va)\Delta = \vggb^2(\va)\Delta
    \end{equation}
    Since $\Delta>0$, we have $\vfg(\va)=\vfgb(\va)$, $\vgg^2(\va)=\vggb^2(\va)$ for all $\va\in \R^D$. 
    From the diagonal diffusion assumption, we know $\vert\vgg(\va)\vert=\vert\vggb(\va)\vert$.

    Now, assume for contradiction that $\mG\neq \mGb$; then there exists $X_{t,i}^\Delta\rightarrow X_{t+1,j}^\Delta$ in $\mG$ but not in $\mGb$. Then we have by definition that $\frac{\partial \bar{f}_j(\mX^\Delta_t, \mGb)}{\partial X^\Delta_{t,i}}= 0$ and $\frac{\partial \bar{g}_j(\mX^\Delta_t, \mGb)}{\partial X^\Delta_{t,i}}= 0$ for all $\mX^\Delta_t$, and also $\frac{\partial f_j(\mX^\Delta_t, \mG)}{\partial X^\Delta_{t,i}}\neq 0$ or $\frac{\partial g_j(\mX^\Delta_t, \mG)}{\partial X^\Delta_{t,i}}\neq 0$ for some $\mX^\Delta_t$.
    In the former case, if $\frac{\partial f_j(\mX^\Delta_t, \mG)}{\partial X^\Delta_{t,i}}\neq 0$ but $\frac{\partial \bar{f}_j(\mX^\Delta_t, \mGb)}{\partial X^\Delta_{t,i}}= 0$ for some $\mX^\Delta_t$, we have a contradiction to $\vfg(\va)=\vfgb(\va)$ for $\va \in \R^D$. A similar analysis can be done in the latter case for $\vgg$, $\vggb$. Thus, we have $\mG=\mGb$, $\vfg = \bar{\vf_G}$ and $\vert\vgg\vert = \vert\bar{\vg}_G\vert$. 
\end{proof}

Next, we will prove a lemma that builds a bridge between the generator of the It\^o diffusion and its corresponding Euler SEM. 
\begin{lemma}[Generator characterises Euler SEM]
    Assume that assumptions \ref{assump: Global Lipschitz} and \ref{assump: diagonal diffusion}. For an It\^o diffusion defined as \cref{eq: Observational process}, we denote its corresponding variables in Euler SEM with $\Delta$ discretization as $\mXd$. Similarly, if we have an alternative It\^o diffusion defined with $\vfgb$, $\vggb$ and $\mGb$, ae define the corresponding Euler SEM variables $\mXdb$. Then, the generators of the It\^o diffusions $\generator = \bar{\generator}$ iff.~their Euler SEM variables have the same distribution with given initial conditions. 
    \label{lem: generator and Euler SEM}
\end{lemma}
\begin{proof}
    First, assume $\generator=\bar{\generator}$, then for any $h\in C_0^2$ (twice continuously differentiable functions vanishing at infinity), we can define the generator for It\^o diffusion as 
    \begin{equation}
        \generator h(\vx) = \sum_d f_d(\vx,\mG)\frac{\partial h(\vx)}{\partial x_d} +\frac{1}{2}\sum_d g^2_d(\vx, \mG)\frac{\partial^2 h(\vx)}{\partial x_d^2}
        \label{eq: ito generator equation}
    \end{equation}
    Similarly, we can define $\bar{\generator}$. From Lemma A.3 \citep{hansen2014causal}, we know if $\generator =\bar{\generator}$, then $\mG=\mGb$, $\vf(\cdot, \mG) = \bar{\vf}(\cdot, \mGb)$ and $\vg^2(\cdot, \mG)=\bar{\vg}^2(\cdot, \mGb)$ for $\vx \in \R^D$. 
    Therefore, by the definition of Euler SEM (\cref{eq: EM updates no graph}), it is trivial that they define the same transition density $p(\mX^\Delta_{t+1}\vert \mX^\Delta_t=\va) = \pb(\mXdb_{t+1}\vert \mXdb_t=\va)$ for $\va\in\R^D$.

    On the other hand, if the two Euler SEMs define the same transition densities, then from Lemma \ref{lem: Identifiability of Euler SEM}, we have $\vfg=\bar{\vf}_G$, $\vert \vgg\vert = \vert\bar{\vg}_G\vert$ and $\mG=\mGb$. Then from \cref{eq: ito generator equation}, $\generator = \bar{\generator}$. 
\end{proof}

Finally, the following lemma shows why we care about the infinitesimal generator for the Feller process.
\begin{lemma}[Generator uniquely determines Feller semigroup]
    Let us define the Feller semigroup transition operator $\transit_t$ and $\bar{\transit}_t$ associated with generator $\generator$, $\bar{\generator}$. Then, $\transit_t=\bar{\transit}_t$ iff.~$\generator=\bar{\generator}$.
    \label{lem: generator determine semigroup}
\end{lemma}
\begin{proof}
    We define the resolvent of a Feller process with $\lambda>0$ as:
    \begin{equation}
        \resolvent_\lambda f = \int_{0}^\infty \exp(-\lambda t)\textstyle{\transit_t}fdt
        \label{eq: resolvent}
    \end{equation}
    with $f\in C_0$. This is the Laplace transform of $\transit_t f$. From \cite{oksendal2003stochastic}, we know $\resolvent_\lambda = (\lambda \mI - \generator)^{-1}$. Therefore, if $\generator = \bar{\generator}$, then for $\lambda>0$, the resolvent $\resolvent_\lambda = (\lambda\mI-\generator)^{-1}=(\lambda \mI-\bar{\generator})^{-1}=\bar{\resolvent}_\lambda$. 
    Therefore, for all $h\in C_0$, they define the same Laplace transform of $\transit_t h$. From the uniqueness of Laplace transform, we have $\transit_t=\bar{\transit}_t$. 

    Similarly, if $\transit_t=\bar{\transit}_t$, we have $\resolvent_\lambda=\bar{\resolvent}_\lambda$ from the definition of resolvent. Thus, $\generator = \lambda\mI - \resolvent_\lambda^{-1}=\lambda\mI-\bar{\resolvent}_\lambda^{-1}=\bar{\generator}$.
\end{proof}

Now, we can prove theorem \ref{thm: structure identifiability observation}. 
\begin{proof}
    Suppose we have two different observation process defined with $\mG\neq \mGb$. Then, from Lemma \ref{lem: Identifiability of Euler SEM}, with any $\Delta>0$, their Euler transition distribution $\pb(\bar{\mX}^\Delta_{t+1}\vert \bar{\mX}^\Delta_t=\va)\neq p({\mX}^\Delta_{t+1}\vert {\mX}^\Delta_t=\va)$. Thus, from Lemma \ref{lem: generator and Euler SEM}, these two It\^o diffusions have different generators $\generator \neq \bar{\generator}$. From assumption \ref{assump: Global Lipschitz}, the solutions of these two It\^o diffusions are Feller processes. From Lemma \ref{lem: generator determine semigroup}, if $\generator \neq \bar{\generator}$, their semigroup $\transit_t\neq \bar{\transit}_t$, which results in different observation distributions of $\mX_t, \mXb_t$. 
\end{proof}

\subsection{Identifiability of latent SDE}
\label{subapp: identifiability of latent sde}
We begin by re-stating \cref{thm: identifiability of latent SDE}:
\IdenLatent*
We follow the same proof strategy as \cite{hasan2021identifying,khemakhem2020variational}.

\begin{proof}
    Let's assume $p(\mX_{t_1},\ldots, \mX_{t_I}) = \pb(\mXb_{t_1},\ldots,\mXb_{t_I})$ even though $\mG\neq \mGb$. Then, for any $t_{i+1}$ and $t_i$, we have $p(\mX_{t_{i+1}},\mX_{t_i}) = \pb(\mXb_{t_{i+1}},\mXb_{t_i})$. Then, we can write 
    \begin{align*}
        p(\mX_{t_{i+1}},\mX_{t_i}) &=\int p(\mZ_{t_{i+1}}, \mZ_\ti, \mX_\tii, \mX_\ti)d\mZ_\tii d\mZ_\ti\\
        &=\int p_z(\mZ_\tii,\mZ_\ti) p_\epsilon(\mX_\tii-\mZ_\tii) p_\epsilon(\mX_\ti-\mZ_\ti)d\mZ_\tii d\mZ_\ti\\
        &=\left[(p_{\cameraready{\epsilon}} \times p_{\cameraready{\epsilon}} ) * p_z\right](\mX_\tii,\mX_\ti)
    \end{align*}
    where $p_\epsilon$ is the noise density for the added observational noise $\bm{\epsilon}$, $p_z$ is the joint density defined by latent It\^o diffusion and $*$ is the convolution operator. Thus, by applying the Fourier transform $\mathcal{F}$, we obtain
    \begin{equation}
        \mathcal{F}(p_\epsilon \times p_\epsilon)(\omega)\times \mathcal{F}(p_z)(\omega) = \mathcal{F}(p_\epsilon \times p_\epsilon)(\omega)\times \mathcal{F}(\pb_z)(\omega)
    \end{equation}
    So $\mathcal{F}(p_z)=\mathcal{F}(\pb_z)$. Then, by inverse Fourier transform, we have $p_z(\mZ_\tii,\mZ_\ti)= \pb_z(\mZb_\tii,\mZb_\ti)$. 

    If the above distributions are obtained by discretizing the It\^o diffusion with a fixed step size $\Delta$, they become the corresponding discretized distribution $p^\Delta(\mZ^\Delta_\tii,\mZ^\Delta_\ti)$ (i.e.~defined by Euler SEM). Then the transition density $p^\Delta(\mZ^\Delta_\tii\vert \mZ^\Delta_\ti) = \pb^\Delta(\mZb^\Delta_\tii\vert \mZb^\Delta_\ti)$. From Lemma \ref{lem: Identifiability of Euler SEM}, we have $\mG=\mGb$, resulting in a contradiction. Thus, $p^\Delta(\mX_{t_1},\ldots,\mX_{t_I})\neq \pb^\Delta(\mXb_{t_1},\ldots, \mXb_{t_I})$. 

    If we have a fixed time range $[0,T]$, then, when we have infinite observations $I\rightarrow \infty$, the observation time $t$ follows an independent temporal point process with intensity $\lim_{dt\rightarrow 0} Pr(\text{observe in }[t,t+dt]\vert \mathcal{H}_t)>0$ where $\mathcal{H}_t$ is the filtration.
    Thus, for arbitrary time interval $\Delta>0$, we have $p(\mZ_{t+\Delta},\mZ_t)=\pb(\mZb_{t+\Delta},\mZb_t)$. Since this holds for arbitrarily small $\Delta>0$, this equality in densities means they define the same transition density $p(\mZ_{t+\Delta}\vert \mZ_t)= \pb(\mZb_{t+\Delta}\vert \mZb_{t})$ as $\Delta\rightarrow 0$. By definition of the Feller transition semigroup, we have $\transit_t=\bar{\transit}_t$. From Lemma \ref{lem: generator determine semigroup}, $\generator = \bar{\generator}$ and $\mG=\mGb$ (Lemma \ref{lem: generator and Euler SEM}, \ref{lem: Identifiability of Euler SEM}). This leads to contradiction, meaning that $p(\mX_{t_1},\ldots, \mX_{t_I})\neq \pb(\mX_{t_1},\ldots, \mX_{t_I})$ when $I \rightarrow \infty$. 
\end{proof}

\subsection{Recovery of the ground truth graph}
\label{subapp: recovery ground truth}
Before diving into the proof of \cref{thm: Consistency}, we introduce some necessary assumptions:
\begin{assumption}[Correctly specified model]
    We say a model is correctly specified w.r.t. the ground truth data generating mechanism iff.~there exists a model parameter such that the model coincides with the generating mechanism. 
    \label{assump: model specification}
\end{assumption}

\begin{assumption}[Expressive posterior process]
    For a given prior parameter $\theta$, we say the approximate posterior process (\cref{eq: posterior process}) is expressive enough if there exists a measurable function $\vu(\mZ_t)$ such that (i) $\vgg(\mZ_t)\vu(\mZ_t)=\vfg(\mZ_t)-\vh_\phi(\mZ, t, \mG)$; (ii) $\vu(\mZ_t)$ satisfies Novikov's condition and (iii) we define
    \begin{equation}
        \mM_T = \exp\left(-\frac{1}{2}\int_0^T \vert \vu(\mZ_t) \vert^2dt-\int_0^T\vu(\mZ_t)^Td\mW_t\right)
        \label{eq: Girsanov M}
    \end{equation}
    and for given latent states $\mZ_{t_1},\ldots, \mZ_{t_I}$ and corresponding observations $\mX_{t_1},\ldots, \mX_{t_I}$ with $0\leq t_1\leq t_2 \leq ...\leq t_I\leq T$, $\mM_T$ can approximate the following arbitrarily well:
    \begin{equation}
        \mM_T\approx \frac{\prod_{i=1}^I p(\mX_{t_i}\vert \mZ_{t_i},\mG)}{p(\mX_{t_1},\ldots,\mZ_{t_I}|\mG)}
    \end{equation}
    \label{assump: expressive posterior}
\end{assumption}
This assumption is to make sure the approximate posterior process is expressive enough to make the variational bound tight. Since we use neural networks to define the drift and diffusion functions, the corresponding approximate posterior is flexible. In fact, \cite{tzen2019theoretical} showed that the diffusion defined by \cref{eq: posterior process} can be used to obtain samples from any distributions whose Radon-Nikodym derivative w.r.t.~standard Gaussian measure can be represented by neural networks. Due to the universal approximation theorem for neural networks \citep{hornik1989multilayer}, the corresponding posterior is indeed flexible. 

First, we can re-write the ELBO (\cref{eq: Overall ELBO}) (for a single time series) as the following: 
\begin{equation}
    \log p(\mX_{t_1},\ldots,\mX_{t_I})\geq \E_{\mG\sim q_\phi(\mG)}\left[\E_{P}\left[\sum_{i=1}^I \log p(\mX_{t_i}\vert \mZt_{t_i})-\frac{1}{2}\int_{0}^T\vert \vu(\mZt_t)\vert^2dt \right]\right] - \KL[q_\phi(\mG)\Vert p(\mG)]
    \label{eq: ELBO}
\end{equation}
where $P$ is the probability measure in the filtered probability space $(\Sigma, \mathcal{F}, \{\mathcal{F}\}_{0\leq t\leq T}, P)$, and $\mZt_t$ is the path sampled from the approximate posterior process (\cref{eq: posterior process}). 
Let's restate the theorem:
\Consistency*

\begin{proof}
    First, we want to show that the term inside the $\E_{\mG\sim q_{\phi}(\mG)}[\cdot]$ represents the $\log p(\mX_{t_1},\ldots,\mX_{t_I}|\mG)$.

    We define a measurable function $\vu(\mZ_t)$ that satisfies Novikov's condition. From the Girsanov theorem, we can construct another process
    \begin{equation}
        d\Wt=\vu(\mZ_t)dt+d\mW_t
        \label{eq: proof Girsanov brownian}
    \end{equation}
    and another probability measure $Q$ s.t.~$\Wt$ is a Brownian motion under measure $Q$ with
    \begin{equation}
        \frac{dQ}{dP}=\exp\left(-\frac{1}{2}\int_0^T \vert \vu(\mZ_t)\vert^2dt-\int_0^T\vu(\mZ_t)^Td\mW_t\right)
        \label{eq: Girsanov dqdp}
    \end{equation}
    where $P$ is the probability measure associated with the original Brownian motion $\mW_t$.
    From \cite{boue1998variational, tzen2019neural}, we have the following variational formulation:
    \begin{equation}
        \log \E_P\left[ \prod_{i=1}^I p(\mX_{t_i}\vert \mZ_{t_i},\mG)\right]=\sup_{Q\in \mathbb{P}}\left\{-\KL[Q\Vert P]+\E_Q\left[\sum_{i=1}^I \log p(\mX_{t_i}\vert \mZ_{t_i},\mG)\right]\right\}
        \label{eq: Girsanov variational formulation}
    \end{equation}
    where $\mathbb{P}$ represents the set of probability measures for the path $\mZ_t$. 
    Assume measure $Q$ is constructed by $\vu$, we can write down $\KL[Q\Vert P]$ by substituting \cref{eq: Girsanov dqdp}:
    \begin{align*}
        \KL[Q\Vert P] &= \E_Q[\log \frac{dQ}{dP}]\\
        &=\int\left[-\frac{1}{2}\int_0^T \vert\vu(\mZ_t)\vert^2dt - \int_0^T \vu(\mZ_t)^Td\mW_t\right]dQ\\
        &=\int\left[-\frac{1}{2}\int_0^T \vert\vu(\mZ_t)\vert^2dt+\int_0^T \vert\vu(\mZ_t)\vert^2dt\right]dQ\\
        &=\E_Q\left[\frac{1}{2}\int_0^T\vert\vu(\mZ_t)\vert^2dt\right].
    \end{align*}
    The third equality can be obtained by manipulating \cref{eq: proof Girsanov brownian}:
    \begin{align*}
        \vu(\mZ_t)^Td\Wt_t&=\vert\vu(\mZ_t)\vert^2dt+\vu(\mZ_t)^Td\mW_t\\
        \Rightarrow \underbrace{\E_Q\left[\int_0^T\vu^Td\Wt_t\right]}_{=0}&=\E_Q\left[\int_0^T\vert\vu\vert^2dt\right]+\E_Q\left[\int_0^T \vu^Td\mW_t\right]
    \end{align*}
    The highlighted term is $0$ due to the martingale property under measure $Q$. Thus, we have 
    \begin{equation}
        \E_Q\left[\int_0^T\vu^Td\mW_t\right] = -\E_Q\left[\int_0^T\vert\vu\vert^2dt\right]
    \end{equation}

    Now, let's define 
    \begin{equation}
        \vu(\mZ_t)=\vgg(\mZ_t)^{-1}[\vf_\theta(\mZ_t,\mG)-\vh_{\phi}(\mZ_t, t,\mG)]
    \end{equation}
    Note that this is different to the original $\vu$ (\cref{eq: u}) by a minus sign. But this does not affect the derivation because we care about $\vu^2$. 
    By simple manipulation of \cref{eq: proof Girsanov brownian}, we have
    \begin{equation}
        \vh_\phi(\mZ_t, t, \mG)dt+\vgg(\mZ_t)d\Wt_t=\vfg(\mZ_t)dt+\vgg(\mZ_t)d\mW_t
    \end{equation}
    This means the prior process (\cref{eq: CRhino formulation}) under probability measure $Q$ is equivalent to the posterior process (\cref{eq: posterior process}) under probability measure $P$. 
    Next, we can change the probability measure of \cref{eq: Girsanov variational formulation}:
    \begin{align*}
        &\sup_{Q\in\mathbb{P}}\left\{\E_Q\left[\sum_{i=1}^I \log p(\mX_{t_i}\vert \mZ_{t_i},\mG)-\frac{1}{2}\int_0^T\vert\vu(\mZ_t)\vert^2dt\right]\right\}\\
        &=\sup_{\vu}\left\{\E_P\left[\sum_{i=1}^I \log p(\mX_{t_i}\vert \mZt_{t_i},\mG)-\frac{1}{2}\int_0^T\vert\vu(\mZt_t)\vert^2dt\right]\right\}
    \end{align*}
    where the second equality is obtained since $\frac{dQ}{dP}$ is fully determined by function $\vu$, and $\mZt_t$ is obtained from the posterior process \cref{eq: posterior process}.  This equation is exactly the term inside $\E_{\mG\sim q(\mG)}[\cdot]$ since $\vgg(\mZ_t)^{-2}[\vfg(\mZ_t)-\vh_\phi(\mZ_t,t,\mG)]^2=\vgg(\mZ_t)^{-2}[\vh_\phi(\mZ_t,t,\mG)-\vfg(\mZ_t)]^2$.

    From Proposition 2.4.2 in \citep{dupuis2011weak}, the supremum is uniquely obtained at 
    \begin{align*}
        \frac{dQ*}{dP}=\frac{\prod_{i=1}^Ip(\mX_{t_i}\vert \mZ_{t_i},\mG)}{\E_P[\prod_{i=1}^Ip(\mX_{t_i}\vert \mZ_{t_i},\mG)]}.
    \end{align*}
    From assumption \ref{assump: expressive posterior}, the measure $Q$ induced by $\vu$ can approximate the above arbitrarily well. Thus, the \cref{eq: ELBO} can be written as:
    \begin{align*}
        \sup_{q(\mG),\theta,\phi} \text{ELBO} &=  \sup_{q(\mG)}\left[\log p(\mX_{t_1},\ldots,\mX_{t_I}\vert \mG) \right]-\KL[q(\mG)\Vert p(\mG)]
    \end{align*}
    We divide the ELBO by $\frac{1}{I}$, and let $I\rightarrow \infty$, we have
    \begin{align*}
        &\lim_{I\rightarrow \infty} \frac{1}{I}\left[\log p(\mX_{t_1},\ldots,\mX_{t_I}\vert \mG) \right]-\frac{1}{I}KL[q(\mG)\Vert p(\mG)]\\
        =& \lim_{I\rightarrow \infty} \frac{1}{I}\left[\log p(\mX_{t_1},\ldots,\mX_{t_I}\vert \mG) \right]\\
        \leq& \lim_{I\rightarrow \infty}\frac{1}{I}\log p(\mX_{t_1},\ldots, \mX_{t_I};\mG^*)
    \end{align*}
    where the first equality is obtained by the fact $\KL[q(\mG)\Vert p(\mG)]<\infty$, and the second inequality is due to the property of the ground truth likelihood. 
    From the identifiability theorem \ref{thm: identifiability of latent SDE}, the equality is uniquely obtained at $q(\mG)=\delta(\mG^*)$, and the learned system recovers the true generating mechanism under infinite data limits. 
\end{proof}

%% file: Appendix_model_arch.tex
\section{Model architecture}
\label{app: model architecture}

In this section, we describe the model architecture details used in our experiments for \ModelName{}.

\paragraph{Prior Drift Function and Diffusion Function} As described in Section \ref{sec: Methodology}, following \cite{geffner2022deep}, we use the following design for the prior drift function $\vf_{G,d}(\mZ_t)$ and diffusion function $\vg_{G,d}(\mZ_t)$:
\begin{equation}
    \vf_{G,d}(\mZ_t) = \zeta\left(\sum_{i=1}^D G_{i,d}l(Z_{t,i}, \ve_{i}), \ve_{d}\right)
\end{equation}
where $\zeta: \mathbb{R}^{D_g\times D_e} \to \mathbb{R}^{D}$, $l: \mathbb{R}^{D \times D_e} \to \mathbb{R}^{D_g}$ are neural networks, and $\ve_i \in \mathbb{R}^{D_e}$ is a trainable node embedding for the $i^{\text{th}}$ node. The use of node embeddings means that we only need to train two neural networks, regardless of the latent dimensionality $D$. 

We implement both the prior drift and diffusion function using $D_e = D_g = 32$, and as neural networks with two hidden layers of size $\max(2*D, D_e)$ with residual connections.

\paragraph{Posterior Drift Function} In Section \ref{subsec: variational inference}, we described the posterior SDE $d\mZt_t^{(n)}= \vh_{\psi}(\mZt_t^{(n)},t; \graph, \mX^{(n)})dt+\vgg(\mZt_t^{(n)})d\mW_t$, with posterior drift function $\vh_{\psi}(\mZt_t^{(n)},t; \graph, \mX^{(n)})$. We now elaborate on how this is implemented. 

We design an encoder $\mK_{\psi}(t, \mG, \mX)$, that takes as input the time $t$, a graph $\mG$ and time series $\mX = \{\mX_{t_1}, ... \mX_{t_I}\}$, and outputs a \emph{context vector} $\vc \in \mathbb{R}^{D_c}$. This encoder consists of a GRU \citep{cho2014learning} that takes as input all future observations (i.e. $\mX_{t_i}$ s.t. $t_i > t$) in reverse order; and a single linear layer which takes the input (i) the hidden state of the GRU, and (ii) the flattend graph matrix $G$, and output the context vector $\vc$. Note that the GRU only takes as input future observations as the future evolution of the latent state is conditionally independent of past observations given the current latent state. We implement the GRU with hidden size $128$, and choose $D_c=64$ for the size of the context vector.

Then, the posterior drift function $\vh_{\psi}(\mZt_t^{(n)},t; \graph, \mX^{(n)})$ is implemented as a neural network that takes as input $\mZt_t^{(n)}$ and the context vector $\vc$ computed by the encoder, and outputs a vector of dimension $D$. This neural network is a MLP with 1 hidden layer of size $128.$

\paragraph{Posterior Mean and Covariance} In Section \ref{subsec: variational inference}, we also have posterior mean and covariance functions $\vmu_\psi(\graph, \mX^{(n)}): \{0, 1\}^{D \times D} \times \mathbb{R}^{D} \to \mathbb{R}^{D}$ and $\mSigma_\psi(\graph, \mX^{(n)}) : \{0, 1\}^{D \times D} \times \mathbb{R}^{D} \to \mathbb{R}^{D \times D}$ for the initial state. We reuse the encoder $\mK_{\psi}(t, \mG, \mX)$ with $t = 0$ to encode the entire time series and graph, and then implement $\vmu_\psi, \mSigma_\psi$ as a linear transformation of the context vector (i.e. a single linear layer). 

\paragraph{Posterior Graph Distribution} In Section \ref{subsec: variational inference}, we introduced a variational approximation $q_{\phi}(\mG)$ to the true posterior $p(\mG|\mX^{(1)}, ..., \mX^{(N)})$. To implement this, we use a product of independent Bernoulli distributions for each edge. That is, we have:
\begin{equation}
    q_{\phi}(\mG) = \prod_{i, j} \phi_{ij}^{G_{i, j}} (1-\phi_{ij})^{(1-G_{i, j})}
\end{equation}
where $\phi_{ij} \in [0, 1]$ are learnable parameters corresponding to the probability of edge $i \to j$ being present. 

\paragraph{Observational Likelihood} \cameraready{We choose the observational noise $p_{\epsilon_d}$ in the model to follow a  standard Laplace distribution with location $\mu = 0$ and scale $b = 0.01$.}

%% file: Baselines.tex
\section{Baselines}
\label{app: baselines}

We use the following baselines for all our experiments to evaluate the performance of \ModelName{}.
\begin{itemize}
    \item PCMCI+:\cite{runge2018causal,runge2020discovering} proposed a constraint-based causal discovery methods for time series, which leverage the momentary conditional independence test to simultaneously detect the lagged parents and instantaneous effects. This is an improvement over its predecessor called PCMCI, which cannot handle instantaneous effects. In our experiments, we use PCMCI for Netsim and PCMCI+ for the other datasets. 
    We use the opensourced implementation \emph{Tigramite} (\url{https://github.com/jakobrunge/tigramite}).
    \item VARLiNGaM: \cite{hyvarinen2010estimation} proposed a linear vector auto-regressive model to learn from time series observations. It is an extension of LiNGaM \citep{shimizu2006linear}, where its structural identifiability is guaranteed through the non-Gaussian noise assumption. The major limitation is its linear and discrete nature, which cannot model complex interactions and continuous systems. We also use the opensourced \emph{LiNGaM} package (\url{https://lingam.readthedocs.io/en/latest/tutorial/var.html})
    \item CUTS: CUTS \citep{cheng2023cuts} is based on Granger causality, and designed for inferring structures from irregularly sampled time series. It treats the irregular samples as a missing data imputation problem. It is capable of imputing missing observations and inferring the graph at the same time. However, it only supports single time series. We use the authors' opensourced code (\url{https://github.com/jarrycyx/unn}). 
    \item Rhino: \cite{gong2022rhino} proposed one of the most flexible SEM-based temporal structure learning framework that is capable of modelling (1) lagged parents; (2) history-dependent noise and (3) instantaneous effects. Many SEM-based structure learning approach can be regarded as a special case of Rhino. From the discussion in \cref{subsec: comparison to Bellot}, \ModelName{} can be regarded as a continous-time version of Rhino. We use the authors' opensourced implementation (\url{https://github.com/microsoft/causica/tree/v0.0.0}). 
    \item NGM: NGM \citep{bellot2021neural} proposed to use NeuralODE to learn the mean process of the SDE. Since this is the only baseline we are aware of in terms of structure learning under continuous time, this will be used as our main comparison. We use the authors' opensourced code (\url{https://github.com/alexisbellot/Graphical-modelling-continuous-time}).
\end{itemize}
NGM and CUTS are originally designed for single time series setup and cannot handle multiple time series. For fair comparison, we modify them by concatenating the multiple time series into a single one. That is, given $n$ time series $\{\vars^{(n)}\}_{n=1}^{N}$ with observation times $t_1, ..., t_I$, we convert them into a single time series with observation times in $[(n-1)*t_I + t_1, n*t_I]$ for the $n^{\text{th}}$ time series.
Our assumption is that since their learning routines are batched across time points, and the concatenation points are rarely sampled, this should have small impact to the performance in comparison to the benefit of additional data. Empirically, this approach indeed improves the performance over simply selecting a single time series. 

For VARLiNGaM, PCMCI, and Rhino, which cannot handle irregularly sampled data, we use zero-order hold (ZOH) to impute the missing data, which has been found to perform competitively \citep{cheng2023cuts} with other imputation methods such as GP regression and GRIN \citep{cini2022filling}.

\subsection{Comparison to ODE-based structure learning}
\label{app: comparison to Bellot}

In this section, we present an extended version of the example failure case of NGM presented in \cref{subsec: comparison to Bellot}. \citet{bellot2021neural} proposed a structure learning method (NGM) for learning from a single time series generated from a SDE. Their approach learns a neural ODE $d\latentvars(\timects) = \driftfn_{\nnparams}(\latentvars(\timects))d\timects$ that models the mean process of the SDE and extract the graphical structure from the first layer of $\vf_\theta$. Given a single observed trajectory $\vars = \{\vars_{\timects_\timeidx}\}_{\timeidx\in[\maxtimeidx]}$, they assume that the observed data follows a multivariate Gaussian distribution $(\vars_{\timects_1}, .. \vars_{\timects_\maxtimeidx}) \sim \mathcal{N}((\latentvars_{\timects_1}, .. \latentvars_{\timects_\maxtimeidx}), \Sigma)$ with mean process $\mZ_t$ given by the deterministic mean process (ODE), and a diagonal covariance matrix $\Sigma \in \mathbb{R}^{\maxtimeidx \times \maxtimeidx}$. As such, NGM optimizes the following squared loss:
\begin{equation} \label{eqn:bellot_loss}
\sum_{\timeidx = 1}^{\maxtimeidx} \lVert \vars_{\timects_i} - \latentvars_{\timects_i} \rVert^2_2
\end{equation}

Like \ModelName{}, NGM attempts to model the underlying continuous-time dynamics and can naturally handle irregularly sampled data. However, the Gaussianity assumption only holds when the underlying SDE is linear; that is, SDEs of the form $d\vars = (\bm{a}(\timects) \vars + \bm{b}(\timects) )d\timects + \bm{c}(\timects) d\wiener_{\timects}$. For general SDEs where the drift and/or diffusion functions are nonlinear functions of the state, the joint distribution can be far from Gaussian, leading to model misspecification, resulting in the incorrect drift function even if the neural network $\driftfn_{\nnparams}$ has the capacity to express the true drift function.  

Another drawback of learning an ODE mean process using the objective in Equation \ref{eqn:bellot_loss} is that it is difficult to generalise to correctly learn from multiple time series, which can be important for recovering the underlying SDEs in practice since a single time series is just a one trajectory sample from the SDE, and thus cannot represent the trajectory multimodality due to stochasticity. 
In particular, simply computing a batch loss over all time series $\sum_{\smpidx=1}^{\maxsmpidx} \sum_{\timeidx = 1}^{\maxtimeidx} \lVert \vars^{(\smpidx)}_{\timects_i} - \latentvars_{\timects_i} \rVert^2_2$ may fail to recover the underlying dynamics when learning from multiple time series. To demonstrate the above argument, we propose a bi-modal failure case. Consider the following 1D SDE:
\begin{equation}
    d\var = \var d\timects + 0.01 d\wienersingle_{\timects}
\end{equation}
where the trajectory will either go upwards or downwards exponentially (bi-modality)

In Figure \ref{fig:failure_case_data_apx} we show trajectories sampled from this SDE, where the initial state is set to $\var_0 = 0$ for all trajectories. The optimal ODE mean process in terms of (batched) squared loss is given by $d\latentvar = 0d\timects$, whose solution is given by the horizontal axis; in particular, while true graph by definition contains a self-loop, the inferred graph from this ODE has no edges. In Figure \ref{fig:failure_case_bellot_apx} we show the ODE mean process $\driftfn_{\nnparams}$ learned by NGM, together with trajectory samples from the corresponding SDE $d\var = \driftfn_{\nnparams}(\var) d\timects + 0.01 d\wienersingle_{\timects}$. The learned ODE mean process (in black) is close to the horizontal axis (note the scale of the vertical axis), with trajectories that do not match the data. On the other hand, in Figure \ref{fig:failure_case_crhino_apx} we see that \ModelName{} successfully learns the underlying SDE with trajectories closely matching the observed data and demonstrating the bi-modal behavior. 

\begin{figure}
    \centering
    \begin{subfigure}{0.32\linewidth}
        \includegraphics[width=\linewidth]{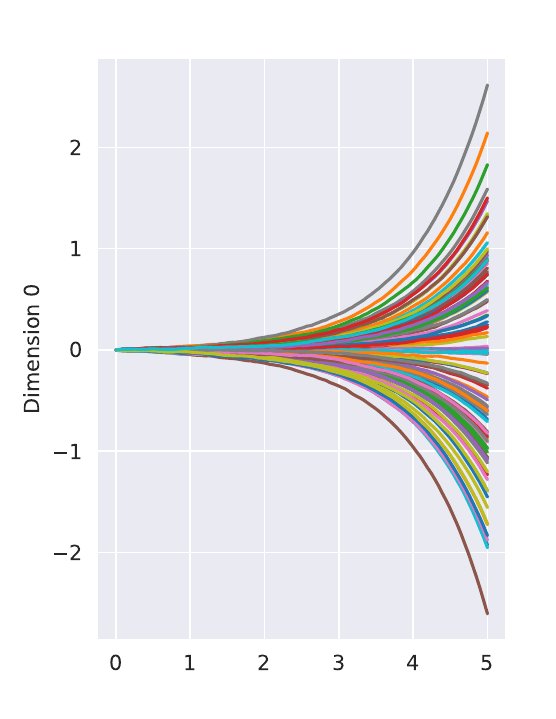}
        \caption{Data}
        \label{fig:failure_case_data_apx}
    \end{subfigure}
    \begin{subfigure}{0.32\linewidth}
        \includegraphics[width=\linewidth]{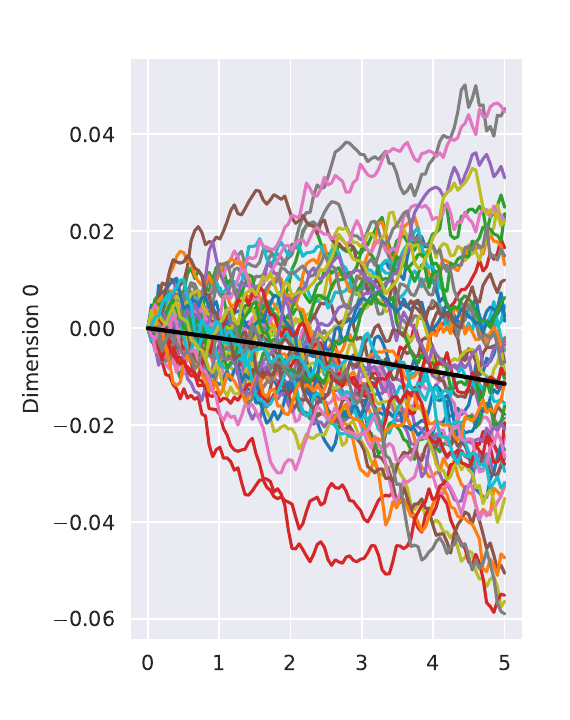}
        \caption{NGM}
        \label{fig:failure_case_bellot_apx}
    \end{subfigure}
    \begin{subfigure}{0.32\linewidth}
        \includegraphics[width=\linewidth]{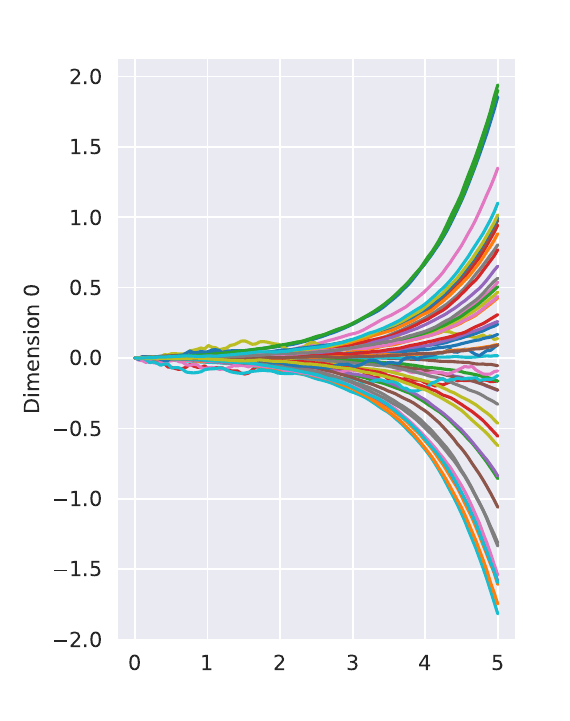}
        \caption{\ModelName{}}
        \label{fig:failure_case_crhino_apx}
    \end{subfigure}
    \caption{Comparison between NGM and \ModelName{} for simple SDE (note vertical axis scale)}
    \label{fig:failure_case_apx}
\end{figure}

%% file: Appendix_experiments.tex
\section{Experiments}
\label{app: Experiments}
\subsection{Choice of SDE solver}
\label{subapp: choice sde solver}
\addition{
There are several choices that can affect the accuracy of the SDE solver used for \ModelName{}. Firstly, discretization step size is an important factor of the solver; a smaller step size generally leads to a more accurate SDE solution, but at the cost of additional time and space complexity. The computational cost (with default Euler discretization) should scale inversely w.r.t.~the step size. In the following, we conducted an initial verification run for the Ecoli1 dataset with half of the original step size reported below in appendix \cref{subapp: dream3}. \Cref{tab: Ecoli1 finer step size} compares the performance with different step sizes. We can see that $\Delta t=0.05$ results in similar performance compared to $\Delta t=0.025$ (while being 2x faster).
Therefore, we decide to use the step size $\Delta t =0.05$. 
Secondly, we chose to use a pathwise gradient estimator rather than the adjoint method \citep{li2020scalable}, as we found this was more efficient time-wise and we did not run into space limitations. Although theoretically, they should give the same performance, in practice, the pathwise gradient estimator may have an advanage that computing its gradient does not require solving another SDE, which is subject to the accuracy of the SDE solver. 
It is also possible to use higher-order numerical solvers such as the Milstein method; however we did not explore this thoroughly in this work.}

\begin{table}[]
\centering
{
\begin{tabular}{l|l}
\hline
                     & \multicolumn{1}{c}{AUROC} \\ \hline
$\Delta t= 0.025$ & \textbf{0.747$\pm$0.005}        \\
$\Delta t=0.05$      & \textbf{0.752$\pm$0.008}     \\ \hline
\end{tabular}}
\label{tab: Ecoli1 finer step size}
\caption{Performance comparisons between different choice of discretization step size $\Delta t$ for SDE solver.}
\end{table}

\subsection{Comparison to Latent SDEs}

\begin{table}[]
\centering
{
\begin{tabular}{l|l}
\hline
\textbf{Method} & \textbf{AUROC} \\ \hline
PCMCI+          & 0.530 $\pm$ 0.002         \\
NGM             & 0.611  $\pm$ 0.002         \\
CUTS            & 0.543  $\pm$ 0.003         \\
Rhino           & 0.685   $\pm$ 0.003        \\
SCOTCH          & \textbf{0.752 $\pm$ 0.008} \\
LSDE            & 0.496 $\pm$ 0.021 \\        
\end{tabular}}
\caption{Performance comparison between methods on DREAM3 Ecoli1 dataset. LSDE refers to latent SDE + extracting first layer weights.}
\label{tab: latent sde graph}
\end{table}

\addition{
Though appealing at first glance, attempting to directly extract graphical structure from SDEs learned using existing methods, such as that of \citet{li2020scalable}, is very challenging. Firstly, to extract the signature graph, one would have to evaluate the partial derivative of the drift and diffusion networks at every input point in the input domain, which is not practical. 
Secondly, the learned drift and diffusion functions may have different graphs, and it is unclear how we should combine these. Thirdly, there are no theoretical results to justify this approach (prior to our paper's theory).
For these reasons, prior work does not admit an easy way to extract structure. \\ \\
In order to construct an simple empirical baseline following this strategy, we follow the setup of \citet{li2020scalable}, and implement each output dimension of the drift and diffusion functions as a separate neural network, i.e. 
\begin{equation}
    \vf = [f_1, ..., f_D]^T, \vg = [g_1, ..., g_D]^T
\end{equation}
Using e.g. $\mA_{g_j}$ to denote the weight matrix of the first layer of $g_j$, and $\mA_{g_j}^k$ to denote the $k^{\textnormal{th}}$ column of that matrix (corresponding to the $k^{\textnormal{th}}$ input dimension, then we define:
\begin{equation}
    \mH_{k, j} = \max(|\mA_{f_j}^k|_2, |\mA_{g_j}^k|_2)
\end{equation}
to be our (weighted) estimate of the graph structure. This has the property that whenever $\mH_{k, j} = 0$, then $\frac{\partial f_j}{\partial x_k} = 0$ and $\frac{\partial g_j}{\partial x_k} = 0$. This can be extracted from a learned SDE, and we can compute an AUROC using the weights as confidence scores. \\ \\
Table \ref{tab: latent sde graph} shows results for this approach (which we call LSDE) in comparison with SCOTCH and other baselines on the DREAM3 Ecoli1 dataset. It can be seen that LSDE performs no better than random guessing at identifying the correct edges.
}

\subsection{Synthetic datasets: Lorenz}

\label{subapp: synthetic lorenz}
\subsubsection{Data generation}
\label{subsubapp: lorenz data generation}

For the Lorenz dataset, we simulate time-series data according to the following SDE based on the $D$-dimensional Lorenz-96 system of ODEs:
\begin{equation}
    d\var_{t, d} = ((\var_{t, d+1} - \var_{t, d-2}) \var_{t, d-1} - \var_{t, d}) dt + F + \sigma dW_{t, i}
\end{equation}
where $\var_{t, -1} := \var_{t, D - 1}, \var_{t, 0} := \var_{t, D}$, and $\var_{t, D + 1}:= \var_{t, 1}$, with parameters set as $F=10$ and $\sigma=0.5$. We generate $N = 100$ $10-$dimensional time series, each with length $I = 100$, which are sampled with time interval $1$ starting from $t = 0$ (that is, $t_1 = 0, t_2 = 1, ..., t_{100} = 99)$. The initial state $\var_{0, i}$ is sampled from a standard Gaussian. To simulate the SDE, we use the Euler-Maruyama scheme with step-size $dt = 0.005$.

For this synthetic dataset, we do not add observation noise to the generated time series.

To produce the irregularly sampled versions of the Lorenz dataset, for each time $t = 0, ..., 99$, we randomly drop the observed data at that time with probability $p$, independently at each time $t$ (and for all time series $n = 1, ... 100$). We test using $p = 0.3, 0.6$ in our experiments.

\subsubsection{Hyperparameters}
\label{subsubapp: lorenz hyperparameters}
\paragraph{\ModelName{}} We use Adam \citep{kingma2014adam} optimizer with learning rate $0.003$ and $0.001$ for $p=0.3$ and $p=0.6$, respectively. We set the $\lambda_s=500$ and EM discretization step size $\Delta=1$ for SDE integrator, which coincides with the step size in the data generation process. The time range is set to $[0,100]$. We enable the residual connections for prior drift and diffusion network. We also adopt a learning rate warm-up schedule, where we linearly increase the learning rate from $0$ to the target value within $100$ epochs. We do not mini-batch across the time series. We train $5000$ epochs for convergence. 

\paragraph{NGM} We use the same hyperparameter setup as NGM \citep{bellot2021neural} where we set $0.1$ for the group lasso regularizer and the learning rate as $0.005$. We train NGM for $4000$ epochs in total ($2000$ for the group lasso stage and $2000$ for the adaptive group lasso stage).

\paragraph{VARLiNGaM} We set the lag to be the same as the ground truth $lag=1$, and do not prune the inferred adjacency matrix. 

\paragraph{PCMCI+} We use \emph{partial correlation} as the underlying conditional independence test. We set the maximum lag at $2$, and let the algorithm itself optimise the significance level. We use the threshold $0.07$ to determine the graph from the inferred value matrix.

\paragraph{CUTS} We use the authors' suggested hyperparameters \citep{cheng2023cuts} for the Lorenz dataset.

\paragraph{Rhino} We use hyperparameters with learning rate $0.01$, $70$ epochs of augmented lagrangian training with $6000$ steps each, time lag of $2$, sparsity parameter $\lambda_s = 5$, and enable instantaneous effects.

\subsubsection{Additional results}
\label{subsubapp: lorenz additional}
\Cref{fig: lorenz additional plot} shows the curve of other metrics. 

\begin{figure}[!h]
\centering
\begin{minipage}[c]{0.45\textwidth}
    \centering
    \includegraphics[scale=0.4]{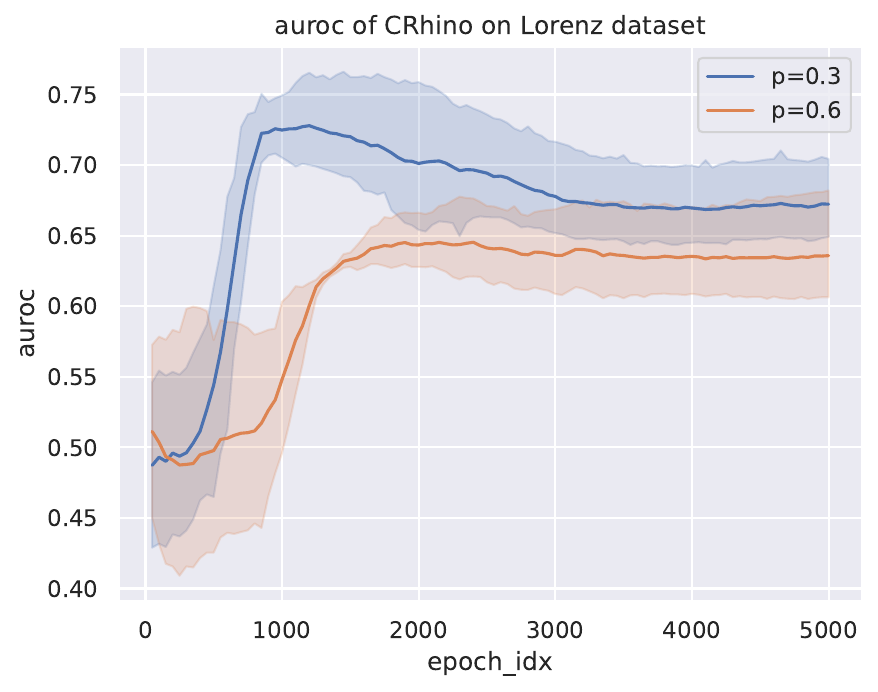}
    \label{fig: lorenz additional plot auroc}
\end{minipage}\hfill
\begin{minipage}[c]{0.45\textwidth}
    \centering
    \includegraphics[scale=0.4]{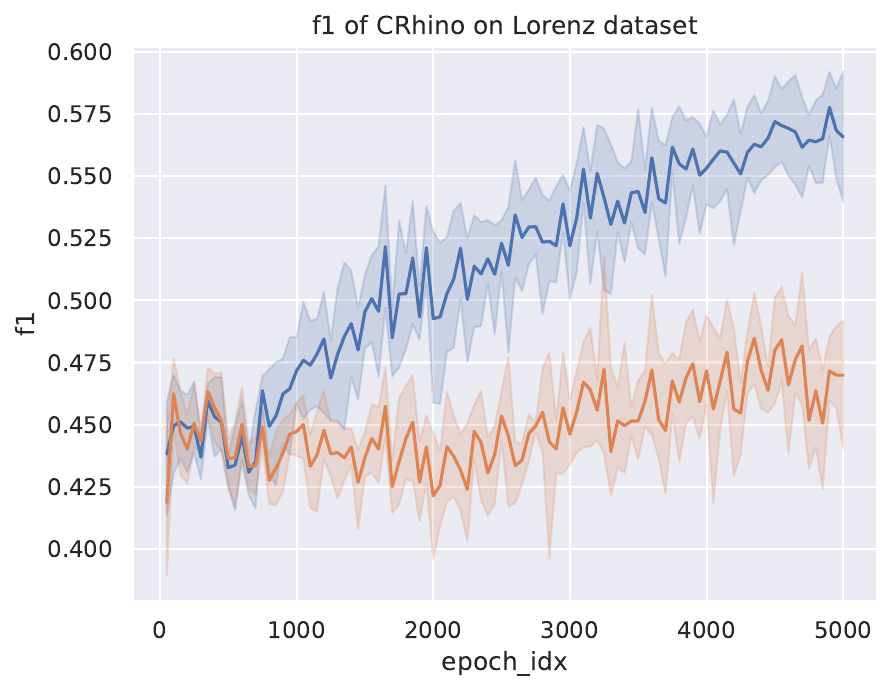}
    \label{fig: lorenz additional plot f1}
\end{minipage}\\
\begin{minipage}[c]{0.45\textwidth}
    \centering
    \includegraphics[scale=0.4]{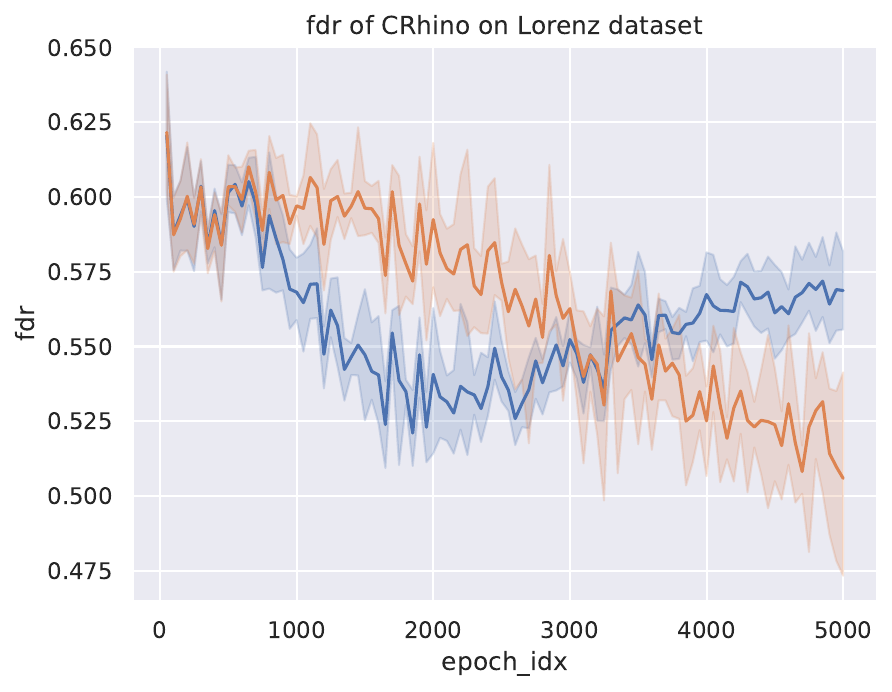}
    \label{fig: lorenz additional plot fdr}
\end{minipage}\hfill
\begin{minipage}[c]{0.45\textwidth}
    \centering
    \includegraphics[scale=0.4]{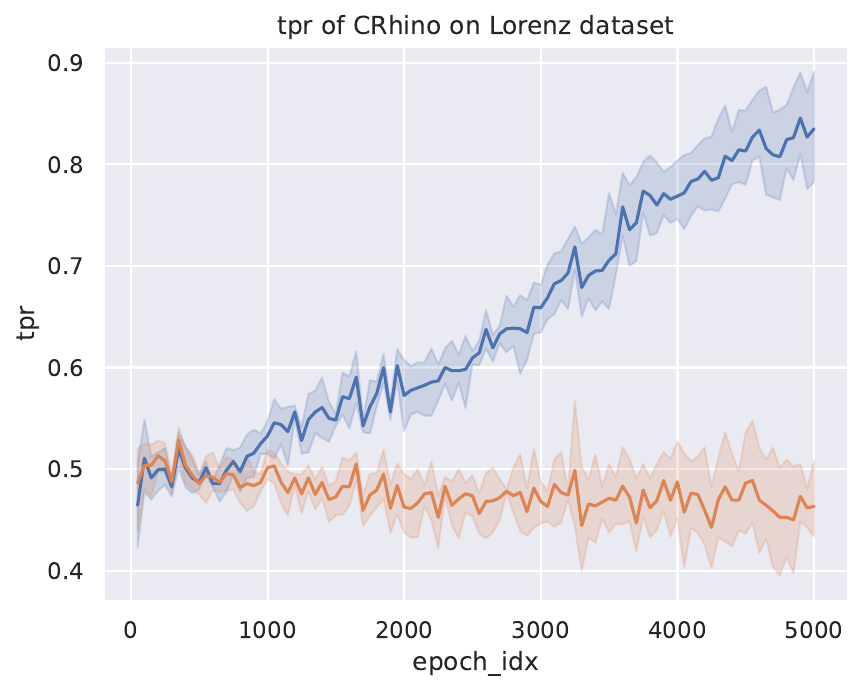}
    \label{fig: lorenz additional plot tpr}
\end{minipage}
    \caption{The AUROC (top left), F1 score (top right), false discovery rate (bottom left) and true positive rate (bottom right) curves of \ModelName{} for Lorenz dataset. The shaded area indicates the $95\%$ confidence intervals. \textcolor{blue}{Blue} color indicates the dataset with missing probability $0.3$ and \textcolor{orange}{orange} color indicates missing probability $0.6$.}
    \label{fig: lorenz additional plot}
\end{figure}

\subsection{Synthetic datasets: Glycolysis}
\label{subapp: synthetic glycolysis}
\subsubsection{Data generation}
\label{subsubapp: glycolysis data generation}

In this synthetic experiment, we generate data according to the system presented by \cite{daniels2015efficient}, which models a glycolyic oscillator. This is a $D=7$ dimensional system with the following equations:
\begin{align*}
    &d\var_{t, 1} = \left(2.5 - \frac{100 \var_{t, 1} \var_{t, 6}}{1 + (\var_{t, 6}/0.52)^4}\right) dt + 0.01 dW_{t, 1} \\
    &d\var_{t, 2} = \left(\frac{200 \var_{t, 1} \var_{t, 6}}{1 + (\var_{t, 6}/0.52)^4} - 6\var_{t, 2}(1 - \var_{t, 5}) - 12\var_{t, 2}\var_{t, 5}\right) dt + 0.01 dW_{t, 2} \\
    &d\var_{t, 3} = \left(6\var_{t, 2} (1 - \var_{t, 5}) - 16\var_{t, 3} (4 - \var_{t, 6})\right) dt + 0.01 dW_{t, 3} \\
    &d\var_{t, 4} = \left(16 \var_{t, 3} (4-\var_{t, 6}) - 100 \var_{t, 4} \var_{t, 5} - 13 (\var_{t, 4} - \var_{t, 7})\right)dt + 0.01 dW_{t, 4} \\
    &d\var_{t, 5} = \left(6 \var_{t, 2} (1 - \var_{t, 5}) - 100 \var_{t, 4} \var_{t, 5} - 12\var_{t, 2}\var_{t, 5} \right)dt + 0.01 dW_{t, 5} \\
    &d\var_{t, 6} = \left(- \frac{200 \var_{t, 1} \var_{t, 6}}{1 + (\var_{t, 6}/0.52)^4} + 32 \var_{t, 3} (4 - \var_{t, 6}) - 1.28 \var_{t, 6}\right) dt + 0.01 dW_{t, 6} \\
    &d\var_{t, 7} = \left(1.3 (\var_{t, 4} - \var_{t, 7}) - 1.8 \var_{t, 7} \right)dt + 0.01 dW_{t, 7}
\end{align*}
As with the Lorenz dataset, we simulate $N = 100$ time series of length $I = 100$, starting at $t = 0$ and with time interval $1$. The initial state is sampled uniformly from the ranges $\var_{0, 1} \in [0.15, 1.60], \var_{0, 2} \in [0.19, 2.16], \var_{0, 3} \in [0.04, 0.20], \var_{0, 4} \in [0.10, 0.35], \var_{0, 5} \in [0.08, 0.30], \var_{0, 6} \in [0.14, 2.67], \var_{0, 7} \in [0.05, 0.10]$, as indicated in \cite{daniels2015efficient}. To simulate the SDE, we use the Euler-Maruyama scheme with step-size $dt = 0.005$.

For this synthetic dataset, we do not add observation noise to the generated time series.

\subsubsection{Hyperparameters}
\label{subsubapp: glycolysis hyperparameter}

\paragraph{\ModelName{}} We use the same hyperparameter as Lorenz experiments. The only differences are that we use learning rate $0.001$ and set $\lambda_s=200$. We train \ModelName{} for $30000$ epochs for convergence. 

\paragraph{NGM} Since \cite{bellot2021neural} did not release the hyperparameters for their glycolysis experiment, we use the default setup in their code. They are the same as the hyperparameters in Lorenz experiments. 

\paragraph{VARLiNGaM} Same as Lorenz experiment setup.

\paragraph{PCMCI+} Same as Lorenz experiment setup.

\paragraph{CUTS} Same as Lorenz experiment setup.

\paragraph{Rhino} Same as Lorenz experiment setup.

\subsubsection{Additional results}
\label{subsubapp: glycolysis additional}

\Cref{tab: glycolysis unnormalized results} shows the performance comparison of \ModelName{} to NGM with the original glycolysis data, where the data have different variable scales. We can observe that this difference in scale does not affect the AUROC of \ModelName{} but greatly affects NGM. Since AUROC is threshold free, we can see that \ModelName{} is more robust in terms of scaling compared to NGM. A possible reason is that the stochastic evolution of the variables in SDE can help stabilise the training when encountering difference in scales, but ODE can easily overshoot due to its deterministic nature. 

\Cref{fig: glycolysis additional plot} shows the curves of different metrics. Interestingly, we can see that data normalisation does not improve the AUROC performance (compared to NGM), but does increase the f1 score. This may be because f1 is threshold sensitive and the default threshold of $0.5$ might not be optimal. 
We can see this through the TPR plot, where \textcolor{orange}{"Original"} has very low value. 
\begin{table}[!h]
\caption{Performance comparison with original Glycolysis data}
\centering
\begin{tabular}{l|lll}
\hline
       & \multicolumn{1}{c}{AUROC} & \multicolumn{1}{c}{TPR $\uparrow$} & \multicolumn{1}{c}{FDR $\downarrow$} \\ \hline
\ModelName{} & \textbf{0.7352$\pm$0.019}          & \textbf{0.3623$\pm$0.007}        & \textbf{0.1575$\pm$0.05}         \\
NGM    &    0.5248$\pm$0.057                       &     0.3478$\pm$0.035                   &     0.4559$\pm$0.094                    \\ \hline
\end{tabular}
\label{tab: glycolysis unnormalized results}
\end{table}

\begin{figure}[!h]
\begin{minipage}[c]{0.45\textwidth}
    \centering
    \includegraphics[scale=0.4]{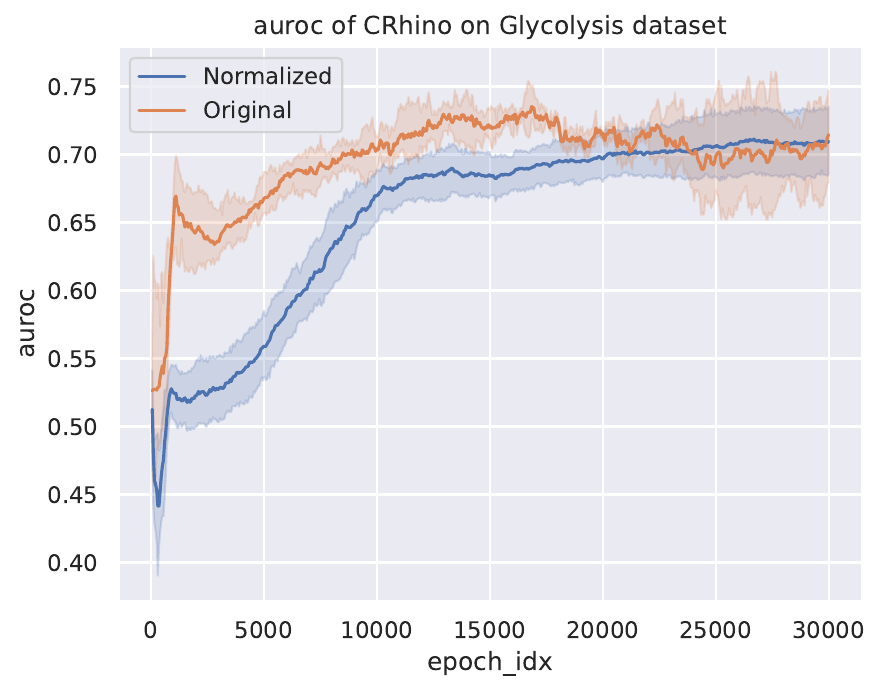}
    \label{fig: glycolysis additional plot auroc}
\end{minipage}\hfill
\begin{minipage}[c]{0.45\textwidth}
    \centering
    \includegraphics[scale=0.4]{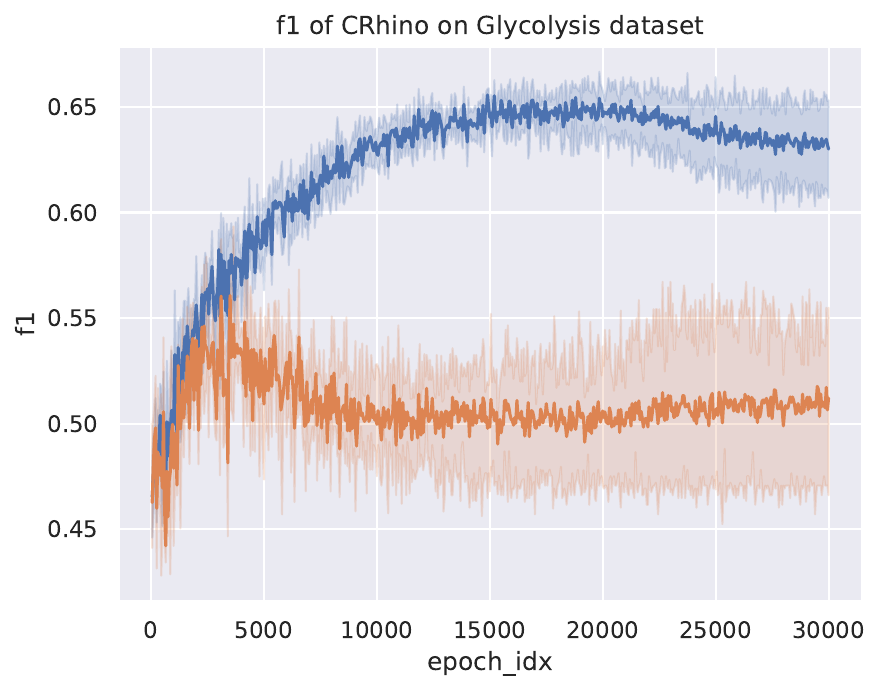}
    \label{fig: glycolysis additional plot f1}
\end{minipage}\\
\begin{minipage}[c]{0.45\textwidth}
    \centering
    \includegraphics[scale=0.4]{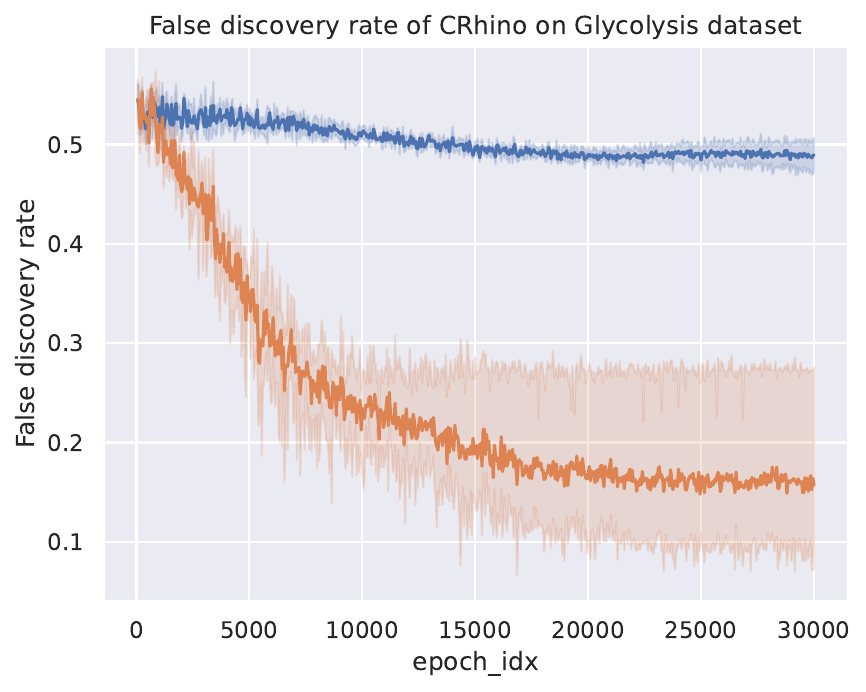}
    \label{fig: glycolysis additional plot fdr}
\end{minipage}\hfill
\begin{minipage}[c]{0.45\textwidth}
    \centering
    \includegraphics[scale=0.4]{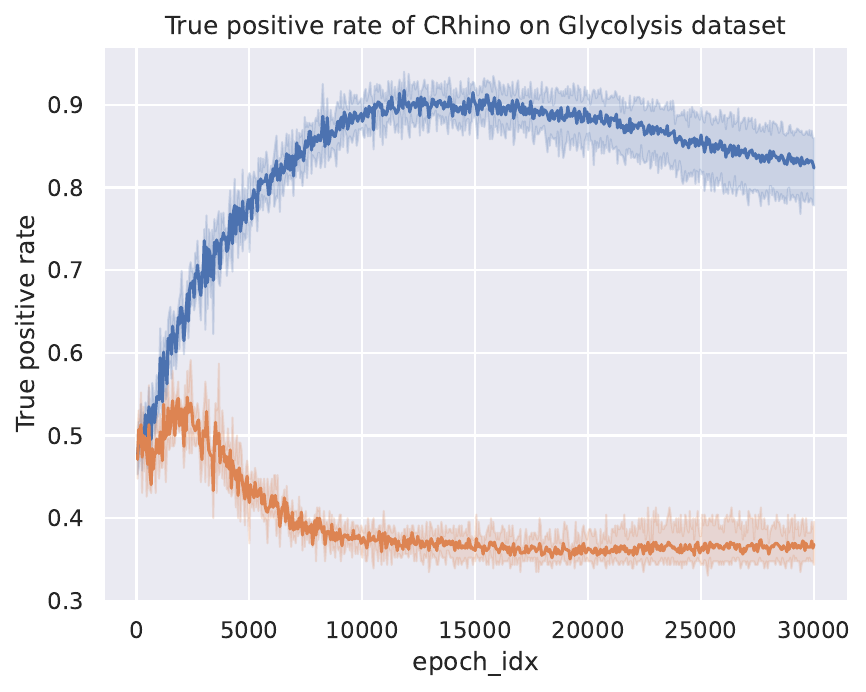}
    \label{fig: glycolysis additional plot tpr}
\end{minipage}
    \caption{The AUROC (top left), F1 score (top right), false discovery rate (bottom left) and true positive rate (bottom right) curves of \ModelName{} for Glycolysis dataset. The shaded area indicates the $95\%$ confidence intervals. \textcolor{blue}{Blue} color indicates the normalized dataset and \textcolor{orange}{orange} color indicates the original dataset.}
    \label{fig: glycolysis additional plot}
\end{figure}

\subsection{Dream3 dataset}
\label{subapp: dream3}
In this appendix, we will include experiment setups, hyperparameters and additional plots for Dream3 experiment.

\subsubsection{Hyperparameters}
\label{subsubapp: dream3 hyperparameters}
\paragraph{\ModelName{}} We follow similar setup as Lorenz experiment. The differences are that the learning rate is $0.001$. The time range is set to $[0,1.05]$ with EM discretization step size $0.05$, which results in exactly $21$ observations for each time series. We choose sparisty coefficient $\lambda_s = 200$. For all sub-datasets, we normalize the data to have $0$ mean and unit variance for each dimension. We use the above hyperparameters for Ecoli1, Ecoli2 and Yeast1 sub-datasets. For Yeast2, we only change the learning rate to be $0.0005$. For Yeast3, we change the $\lambda_s=50$. We train \ModelName{} for $30000$ epochs until convergence.

\paragraph{NGM} For NGM, we follow the same hyperparameter setup as \citep{cheng2023cuts}, where we set the group lasso regulariser as $0.05$, learning rate $0.005$. We train NGM with $4000$ epochs ($2000$ each for group lasso and adaptive group lasso stages). For fair comparison, we use the same observation time (i.e.~equally spaced time points within $[1,1.05]$ and step size $0.05$). 

\paragraph{PCMCI+ and Rhino} As the experiment setup is the same, we directly cite the number from \cite{gong2022rhino}.

\paragraph{CUTS} We use the authors' suggested hyperparameters \citep{cheng2023cuts} for the DREAM3 datasets.

\subsubsection{Additional plots}
\label{subsubapp: dream3 additional plots}
In this section, we include additional metric curves of \ModelName{} in \cref{fig: dream3 additional plot}. Each curve is obtained by averaging over 5 runs and the shaded area indicates the $95\%$ confidence interval. From the value of f1 score, FDR and TPR, we can see DREAM3 is indeed a challenging dataset, where all f1 scores are below 0.5 and FDR only drops to 0.7. From the TPR plot, it is expected to drop at the beginning and then increase during training, which is the case for Ecoli1, Ecoli2 and Yeast1. TPR corresponds well to AUROC and F1 score, since Ecoli1, Ecoli2 and Yeast1 have much better values compared to Yeast2 and Yeast3. 
\begin{figure}[!h]
\begin{minipage}[c]{0.45\textwidth}
    \centering
    \includegraphics[scale=0.4]{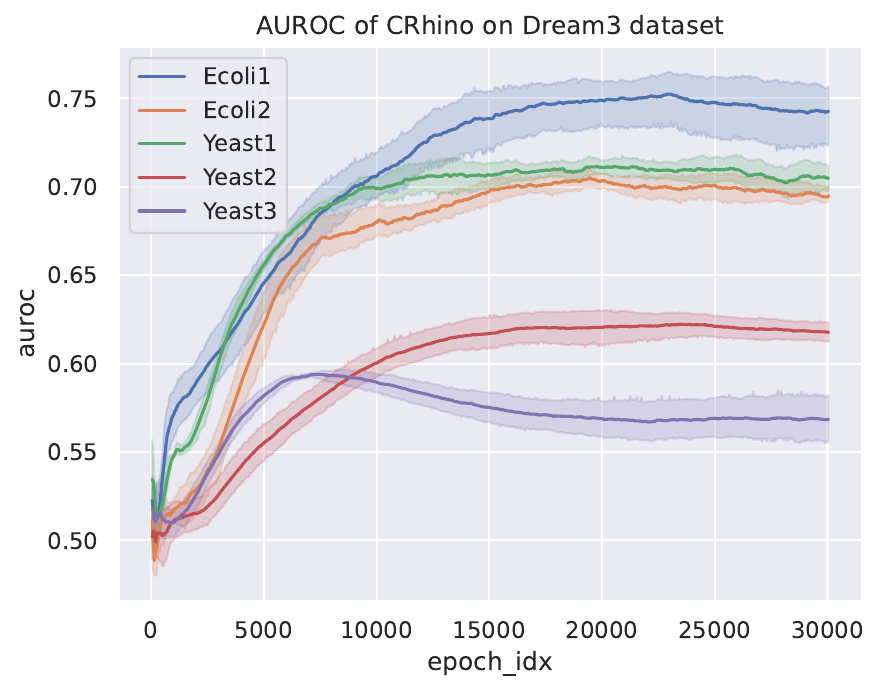}
    \label{fig: dream3 additional plot auroc}
\end{minipage}\hfill
\begin{minipage}[c]{0.45\textwidth}
    \centering
    \includegraphics[scale=0.4]{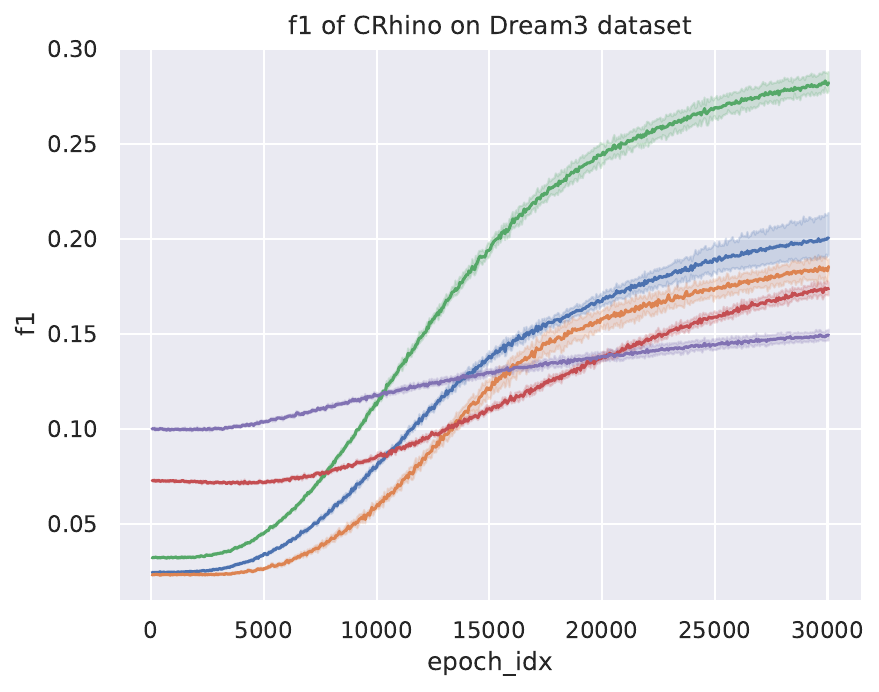}
    \label{fig: dream3 additional plot f1}
\end{minipage}\\
\begin{minipage}[c]{0.45\textwidth}
    \centering
    \includegraphics[scale=0.4]{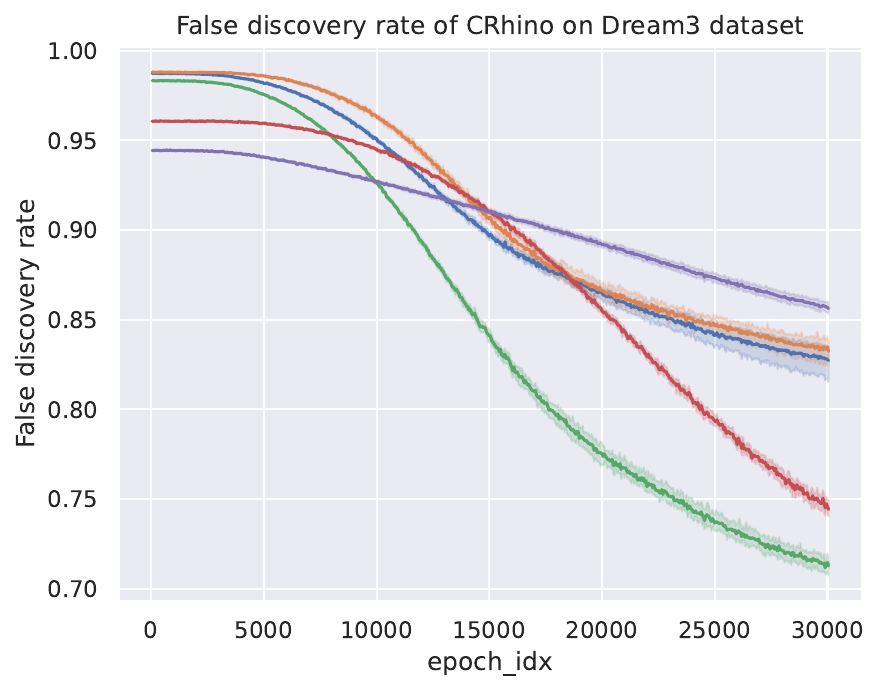}
    \label{fig: dream3 additional plot fdr}
\end{minipage}\hfill
\begin{minipage}[c]{0.45\textwidth}
    \centering
    \includegraphics[scale=0.4]{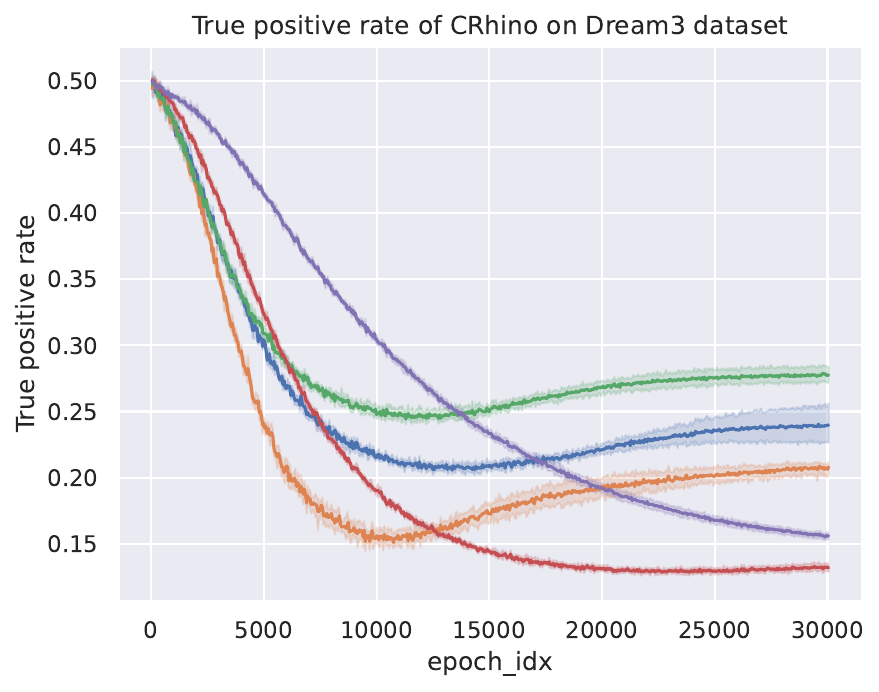}
    \label{fig: dream3 additional plot tpr}
\end{minipage}
    \caption{The AUROC (top left), F1 score (top right), false discovery rate (bottom left) and true positive rate (bottom right) curves of \ModelName{} for each DREAM3 sub-datasets. The shaded area indicates the $95\%$ confidence intervals. }
    \label{fig: dream3 additional plot}
\end{figure}

\subsection{Netsim}
\label{subapp: netsim}
\subsubsection{Experiment setup}
\label{subsubapp: netsim experiment setup}

For the Netsim dataset, we generate the missing data versions in the same way as the Lorenz dataset (see appendix \ref{subapp: synthetic lorenz}). 

\subsubsection{Hyperparameters}
\label{subsubapp: netsim hyperparameter}
\paragraph{\ModelName{}} We use similar hyperparameter setup as Dream3 (\cref{subsubapp: dream3 hyperparameters}), but we change $\lambda_s=1000$ and use the raw data without normalisation. We train \ModelName{} for $10000$ epochs. 

\paragraph{NGM} We follow the same setup as DREAM3 experiment, which also coincides with the setup used in \cite{cheng2023cuts}.  

\paragraph{PCMCI} We follow the same setup as Lorenz and use threshold $0.07$ to infer the graph. 

\paragraph{CUTS} We use the authors' suggested hyperparameters \citep{cheng2023cuts} for the Netsim dataset.

\paragraph{Rhino and Rhino+NoInst} We directly cite the number from \cite{gong2022rhino} for the full dataset, and use the same hyperparameters as \cite{gong2022rhino} for both $p=0.1$ and $p=0.2$ Netsim datasets.

\subsubsection{Additional plots}
\label{subsubapp: netsim additional plots}
We include additional metric curves of \ModelName{} on Netsim dataset in \cref{fig: netsim additional plot}.
From the plot, we can see Netsim is a easier dataset compared to DREAM3 since the dimensionality is much smaller. An interesting observation is f1 score does not necessarily correspond well to auroc since f1 score is threshold dependent (by default we use 0.5) but not auroc. To evaluate the robustness of the model, we decide to report AUROC instead of f1 score. 

\begin{figure}[!h]
\begin{minipage}[c]{0.45\textwidth}
    \centering
    \includegraphics[scale=0.4]{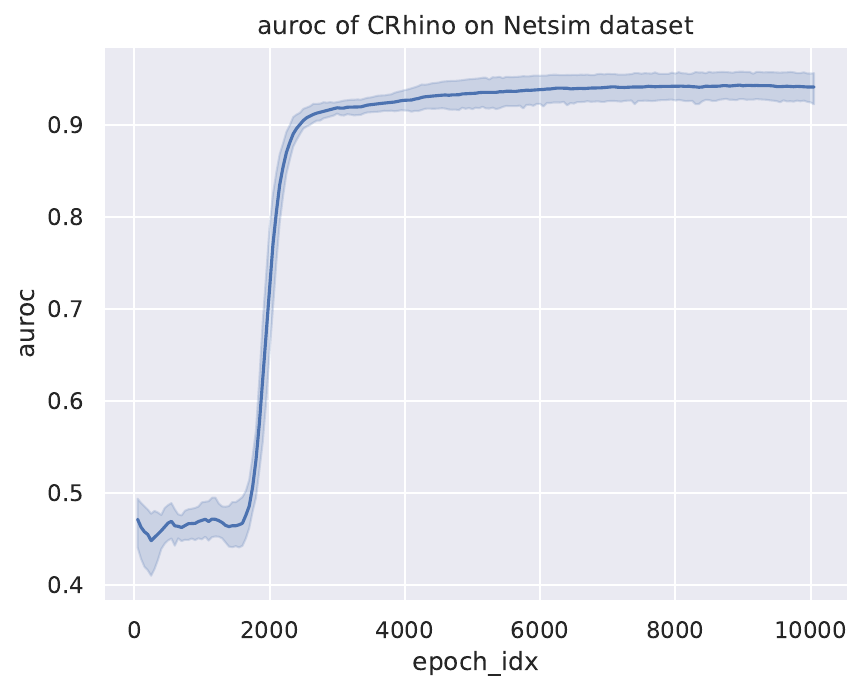}
    \label{fig: netsim additional plot auroc}
\end{minipage}\hfill
\begin{minipage}[c]{0.45\textwidth}
    \centering
    \includegraphics[scale=0.4]{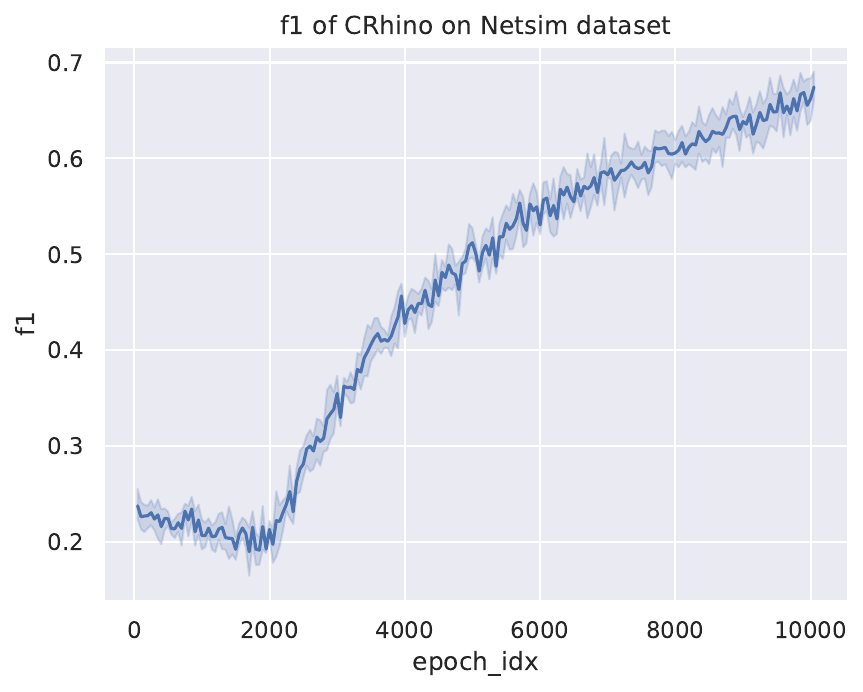}
    \label{fig: netsim additional plot f1}
\end{minipage}\\
\begin{minipage}[c]{0.45\textwidth}
    \centering
    \includegraphics[scale=0.4]{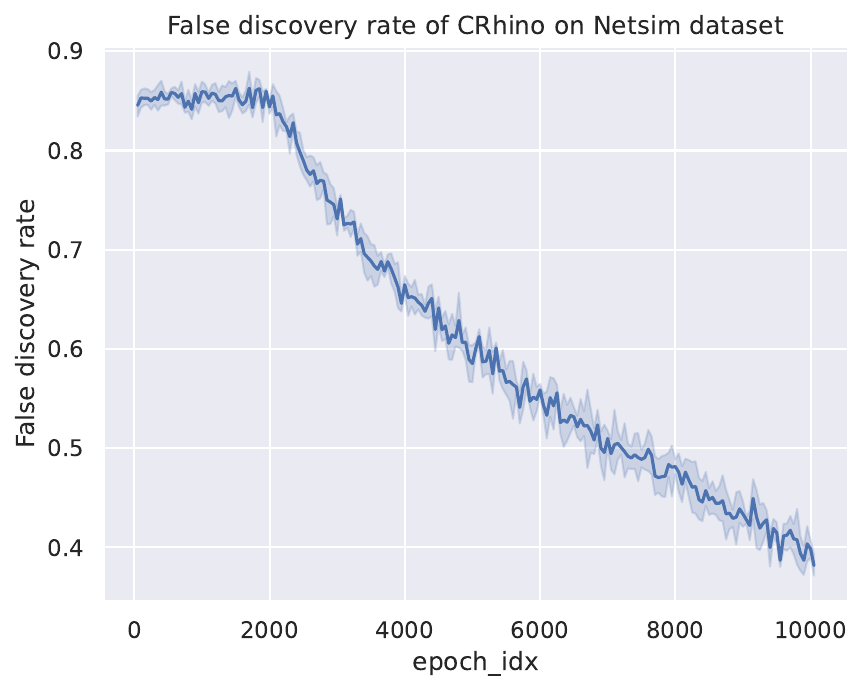}
    \label{fig: netsim additional plot fdr}
\end{minipage}\hfill
\begin{minipage}[c]{0.45\textwidth}
    \centering
    \includegraphics[scale=0.4]{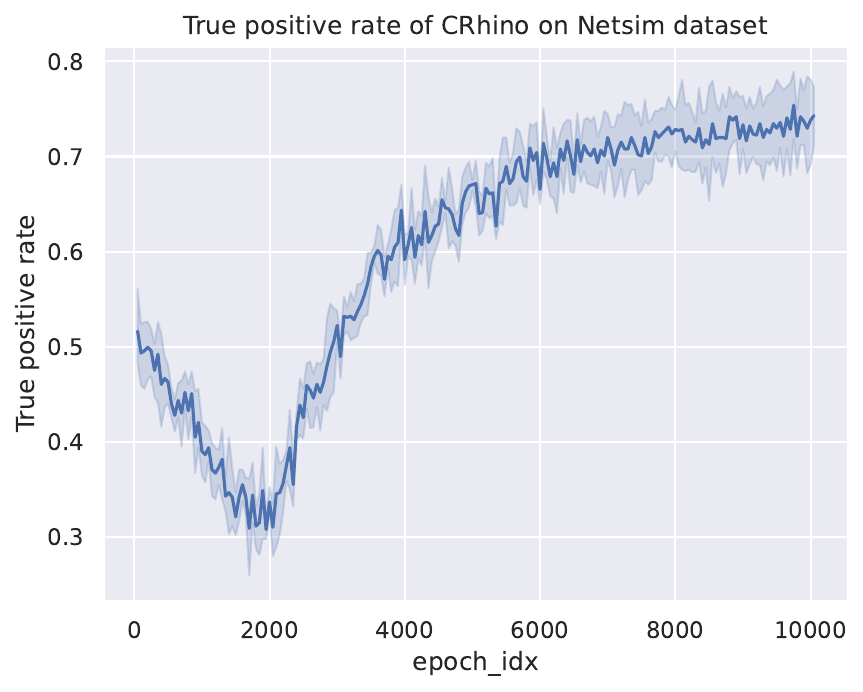}
    \label{fig: netsim additional plot tpr}
\end{minipage}
    \caption{The AUROC (top left), F1 score (top right), false discovery rate (bottom left) and true positive rate (bottom right) curves of \ModelName{} for Netsim dataset. The shaded area indicates the $95\%$ confidence intervals. }
    \label{fig: netsim additional plot}
\end{figure}

%% file: Appendix_interventions.tex
\section{Interventions}
\label{app: intervention}

Aside from learning the graphical structure between variables, one might also be interested in analysing the effect of applying external changes, or interventions, to the system. Broadly speaking, there are two types of interventions that we can consider in a continuous-time model. The first is to intervene on the dynamics (that is, the drift or diffusion functions), possibly for a set period of time. 
The second is to directly intervene on the value of (some subset of) variables. The goal is to employ our learned \ModelName{} model in order to predict the effect of these interventions on the underlying system. 

The former is easy to implement as we need only replace (parts of) the learned drift/diffusion function with the intervention. However, the latter is slightly more subtle than it might first appear. \citep{hansen2014causal} proposed to define such an intervention as a function that fixes the value of a particular variable as a function of the other variables. However, it is unclear how we can generalize this to interventions affecting more than one variable. For example, a intervention policy $\latentvar_1 \gets \latentvar_2, \latentvar_2 \gets \latentvar_1 + 1$ creats a feedback loop whose semantics are not easy to resolve. Thus, we propose the following definition:

\begin{definition}[State-space Intervention]
    Given a $\maxvaridx$-dimensional SDE, a state-space intervention is an idempotent function $\intv(t, \latentvars): \mathbb{R}^{\maxvaridx + 1} \to \mathbb{R}^\maxvaridx$; that is, $\intv(t, \intv(t, \latentvars)) \equiv \intv(t, \latentvars)$. The corresponding intervened stochastic process is defined by:
    \begin{equation}
        \tilde{\latentvars}_t = \iota\left(t, \tilde{\latentvars}_0 + \sum_{\varidx \in [\maxvaridx]} \int_{0}^{t} \vf(\tilde{\latentvars}_s) ds +  \sum_{\varidx \in [\maxvaridx]} \int_{0}^{t} \vg(\tilde{\latentvars}_s) d\wiener_s \right)
    \end{equation}
\end{definition}

The requirement of idempotence captures the intuition that applying the same intervention twice should result in the same result. Some examples of interventions are given as follows:
\begin{itemize}
    \item \textbf{Identity}: If $\intv(t, \latentvars) = \latentvars \; \forall t \in [T_1, T_2], \latentvars \in \mathbb{R}^{\maxvaridx}$, then the system evolves accoridng to the original SDE in this time period, with initial state $\tilde{\latentvars}_{T_1}$.
    \item \textbf{Ordered Intervention}: Given some ordered subset of the variables, we can consider intervening on each variable in order, as a function of the previous variables in the order. That is, we restrict each dimension $\intv_i$ of the intervention output to be of the form
    \begin{equation}
        \intv_i(t, \latentvars_t) = \intv_i(t, \latentvars_{t, <i})
    \end{equation}
    where $\latentvars_{t, <i} = \{\latentvars_{t, j} : j < i\}$. It can easily be seen that $\intv$ is always idempotent in this case.
    \item \textbf{Projection}: Another example of an idempotent function is a projection. This could simulate a setting where external force is applied to ensure the SDE trajectories satisfy spatial constraints. Note that a projection cannot necessarily be expressed as an ordered intervention (e.g. consider projection onto a sphere).
\end{itemize}

In practice, we implement state-space interventions in SDEs learned from \ModelName{} by modifying the SDE solver (e.g. Euler-Maruyama) such that each step is followed with an intervention assignment $\latentvars_t \gets \intv(t, \latentvars_t)$.